%% file: main.tex
\PassOptionsToPackage{table}{xcolor} % 在 acl2026.sty 加载 xcolor 之前传递 table 选项
\documentclass[11pt]{article}
\usepackage[final]{acl2026}
\usepackage{times}
\usepackage{latexsym}
\usepackage[T1]{fontenc}
\usepackage[utf8]{inputenc}

\usepackage{microtype}      % 优化排版微调
\usepackage{booktabs}       % 专业表格线
\usepackage{graphicx}       % 图形支持
\usepackage{subcaption}     % 子图支持
\usepackage{titletoc}       % 目录格式
\usepackage{enumitem}       % 列表定制
\usepackage{float}          % 浮动体控制 (H)
\usepackage{multirow}       % 表格多行合并
\usepackage{pifont}         % 特殊符号 (如 dingbat)
\usepackage{xspace}         % 智能空格处理

\usepackage{amsmath}
\usepackage{amssymb}
\usepackage{mathtools}
\usepackage{amsthm}

\usepackage[listings,skins,breakable]{tcolorbox} % 彩色文本框

\usepackage{algorithm}
\usepackage{algpseudocode}

\usepackage{hyperref} 
\usepackage[textsize=tiny]{todonotes}           % 开发时使用

\usepackage[capitalize,noabbrev]{cleveref}

\usepackage{listings}
\lstdefinestyle{python}{
    language=Python,                 % 语言设置为 Python
    basicstyle=\ttfamily\footnotesize, % 字体样式 (等宽, 小号)
    keywordstyle=\color{blue}\bfseries, % 关键字样式 (蓝色加粗)
    commentstyle=\color{gray},       % 注释样式 (灰色)
    stringstyle=\color{red},         % 字符串样式 (红色)
    breaklines=true,                 % 自动换行
    frame=single,                    % 给代码加边框 (可选: none, single, shadow)
    numbers=left,                    % 行号显示在左侧 (可选: none)
    numberstyle=\tiny\color{gray},   % 行号样式
    showstringspaces=false,          % 不显示字符串中的空格标记
    captionpos=b,                    % 标题位置
    backgroundcolor=\color{white},   % 背景色
}
\algnewcommand\algorithmicinput{\textbf{Input:}}
\algnewcommand\Input{\item[\algorithmicinput]}
\algnewcommand\algorithmicoutput{\textbf{Output:}}
\algnewcommand\Output{\item[\algorithmicoutput]}
\newcommand{\algcomment}[1]{\State \textcolor{blue}{// #1}} % 蓝色注释

\theoremstyle{plain}
\newtheorem{theorem}{Theorem}[section]

\theoremstyle{definition}
\newtheorem{definition}[theorem]{Definition}
\newtheorem{assumption}[theorem]{Assumption}
\theoremstyle{remark}

\definecolor{mygray}{gray}{.9}
\definecolor{ggray}{RGB}{127,127,127}
\definecolor{reda}{RGB}{192,0,0}
\definecolor{redb}{RGB}{217,148,143}
\definecolor{myyellow}{RGB}{190,144,0}
\definecolor{myblue}{RGB}{30,90,100}
\definecolor{tabhighlight}{HTML}{e5e5e5}

\newcommand{\red}[1]{\textcolor{reda}{\textbf{#1}}}   % 暂时简化为仅变色，如有 mytextbox 定义请恢复
\newcommand{\blue}[1]{\textcolor{myblue}{\textbf{#1}}}
\newcommand{\green}[1]{\textcolor{green!50!black}{\textbf{#1}}}

\newcommand{\purple}[1]{\textcolor{purple}{\textbf{#1}}}

\newcommand{\ourapproach}{$\text{E}^3$-TIR\xspace}

\definecolor{reda}{RGB}{192,0,0}
\usepackage{fontawesome5} 
\usepackage{seqsplit}

\title{$\text{E}^3$-TIR:  Enhanced Experience Exploitation for  Tool-Integrated Reasoning}

\author{
    Weiyang Guo$^1$
    \quad Zesheng Shi$^1$
    \quad Liye Zhao$^2$
    \quad Jiayuan Ma$^1$\\
    \quad \textbf{Zeen Zhu}$^1$
    \quad \textbf{Junxian He}$^3$
    \quad \textbf{Min Zhang}$^1$
    \quad  \textbf{Jing Li}$^1$\textsuperscript{\texorpdfstring{\faIcon[regular]{envelope}}{}} 
    \\$^{1}$Harbin Institute of Technology, Shenzhen, China 
    \\$^{2}$Huawei Technologies Co., Ltd.  
    \\$^{3}$Hong Kong University of Science and Technology  \\
    \texttt{weiyangguo2024@gmail.com} \quad \texttt{jingli.phd@hotmail.com}  
}

\begin{document}
\maketitle

\begin{abstract}

While Large Language Models (LLMs) have demonstrated significant potential in Tool-Integrated Reasoning (TIR), existing training paradigms face significant limitations: Zero-RL suffers from inefficient exploration and mode degradation due to a lack of prior guidance, while SFT-then-RL is limited by high data costs and capability plateaus caused by low-entropy collapse. To address these challenges, we propose \ourapproach (Enhanced Experience Exploitation), a warm-up paradigm for the early stages of agent training. Specifically, we formulate training as the dynamic integration of three experience types: Expert Prefixes, Expert Guided, and Self-Exploration. By executing diverse branching exploration around expert ``anchors’’ and employing a mix policy optimization mechanism, we effectively mitigate distribution shifts and resolve optimization conflicts arising from shared prefixes. Our method dynamically adapts the model’s knowledge boundaries, effectively balancing exploration diversity with training efficiency. Experimental results demonstrate that \ourapproach achieves a 6\% performance improvement over traditional paradigms on tool-use tasks, while requiring less than 10\% of the synthetic data. Furthermore, in terms of ROI—a comprehensive metric integrating performance, data cost, and training efficiency—we achieve a 1.46$\times$ gain compared to baselines.  Code is available at \url{https://github.com/yuki-younai/E3-TIR}.
\let\thefootnote\relax\footnotetext{\faIcon[regular]{envelope}~Corresponding author.}
\end{abstract}

\section{Introduction}
\label{introduction}

Large Language Models (LLMs) have demonstrated exceptional potential in complex reasoning and tool-assisted decision-making, which has driven the realization of advanced applications such as question answering, planning, and interactive agents~\cite{survey1,survey2}. To overcome the limitations of purely internal reasoning, LLMs are often deployed as Tool-Integrated Reasoning (TIR) agents~\cite{dong2025aepo}. A TIR agent performs reasoning and action through a multi-step, interleaved process that invokes external tools such as code interpreters~\cite{xue2025simpletir} and search engines~\cite{ragen}.

Recently, reinforcement Learning (RL) has shown significant promise in unlocking frontier LLM capabilities~\cite{deepseekmath2024}, methods such as Search-R1~\cite{search-r1} and ToRL~\cite{torl} train directly on the foundation model, utilizing self-exploration to acquire tool-use capabilities, but this compromises training efficiency. Conversely, approaches like Tool-Star~\cite{toolstar} and Tool-Light~\cite{toollight} necessitate an initial stage of supervised fine-tuning with synthetic data before engaging in reinforcement learning. Reliance on synthetic data is not only computationally expensive but also introduces training redundancy.

\begin{figure}[t]
	\centering
	\includegraphics[width=0.95\columnwidth]{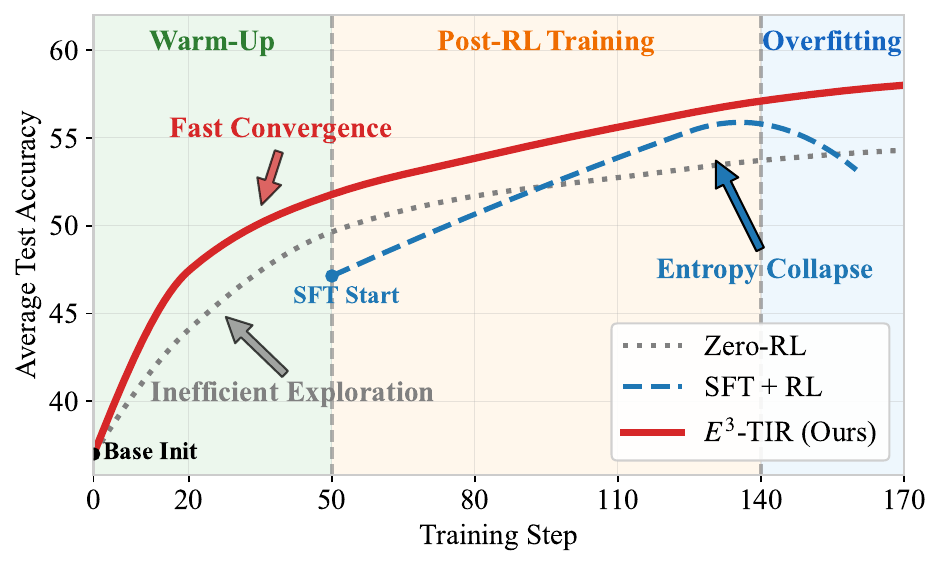}e
    	\caption{(a) Zero-RL:  Inefficient exploration and slow convergence. (b) SFT+RL: Vulnerable to overfitting and performance collapse after a high start. (c) \ourapproach: Enables rapid adaptation via expert guidance with stable, continuous growth.  }
        \label{fig1:motivation}
\end{figure}

\begin{figure*}[th]
    \centering \includegraphics[width=0.98\linewidth, trim=0 0 0 0, clip]{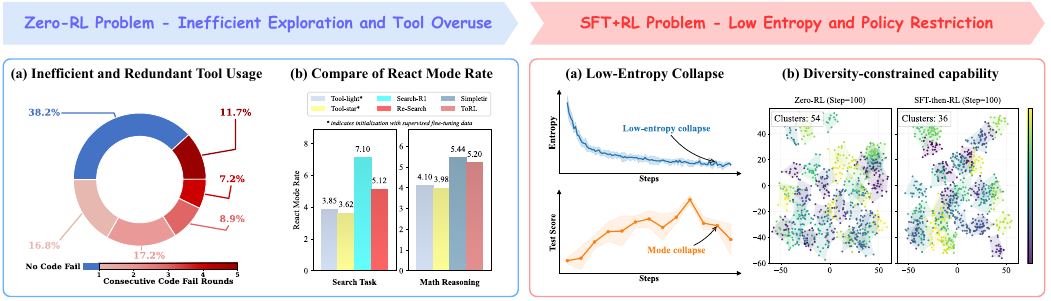}
    \caption{Statistical analysis of the limitations of current training paradigms.}
    \label{fig:motivation}
\end{figure*}

Our empirical analysis further reveals the deep-seated limitations of existing paradigms. Zero-RL, lacking prior guidance, exhibits inefficient exploration and mode degradation: the model requires extensive trial-and-error to generate valid code at an early stage (Fig.~\ref{fig:motivation} left,a) and often converges to a ``react mode’’~\cite{demystifying} that over-relies on tools at the expense of reasoning (Fig.~\ref{fig:motivation} left, b). Conversely, SFT-then-RL, while benefiting from training, leads to capability rigidity. This necessitates an adaptation period to overcome distribution shifts (Fig.~\ref{fig:motivation} right,a) and frequently triggers low-entropy collapse in the later stages of training (Fig.~\ref{fig:motivation} right,b), which severely restricts sample diversity and impedes further performance gains. 

To address these challenges, we propose \ourapproach, a framework designed to balance exploration diversity with efficiency. In the sampling stage,we initiate branching sampling from high-entropy ``anchors’’ within expert trajectories to explore diverse reasoning paths. Subsequently, we dynamically filter the mixed trajectories based on gradient validity and relative performance.  For mix policy optimization, we employ hybrid advantage estimation to capture preference variances among branches sharing the same prefixes. Meanwhile, advantage blocking resolves
gradient conflicts on shared prefixes, ensuring training stability.

 The main contributions of this study are as follows:
 \begin{itemize}[noitemsep,nolistsep]
 \item
Through empirical analysis, we uncover the ``react mode’’  inherent in Zero-RL tool call, alongside the phenomena of low-entropy collapse and capability rigidity that occur in the later stages of SFT+RL.
 \item  We propose \ourapproach, which provides dynamic experience sources for an agent’s early training by integrating three distinct types of experience: expert prefixes, expert guidance, and self-exploration. This approach effectively balances the diversity and efficiency of early-stage exploration. 
 \item \ourapproach outperforms SFT and RL baselines using <10\% of the synthetic data. With a minimal warm-up, it achieves a 1.46$\times$ gain in ROI—a composite metric of data, training, and performance efficiency.
\end{itemize}

\section{Preliminaries}
\subsection{Multi-Turn Agent Framework}

We consider an LLM agent performing a task $x \in D$ through iterative interactions. At each step $t$, the policy $\pi_{\theta}(\cdot|\mathcal{H}_{<t})$ generates a thought $\tau_t$ and an action $a_t$, eliciting an observation $o_t$ from the environment. This yields a trajectory $\mathcal{H} = \{(\tau_t, a_t, o_t)\}_{t=1}^T$. As the trajectory length grows linearly with the number of steps, optimizing $\pi_{\theta}$ for complex tasks becomes increasingly challenging due to the accumulation of history.
\subsection{Agentic Reinforcement Learning}

Agentic RL typically adopts policy-gradient methods to optimize the agent policy $\pi_{\theta}$. We formulate the agentic RL training objective as: 
\begin{equation}
\max_{\pi_\theta} \mathbb{E}_{x \sim \mathcal{D}, \mathcal{H}  \sim \pi_\theta(\cdot|x)} \left[ r_\phi(x, \mathcal{H} ) \right] - \beta \mathbb{D}_{\text{KL}} \left[ \pi_\theta(\cdot) \| \pi_{\text{ref}}(\cdot) \right]
\end{equation}
where $\pi_{\theta}$ represents the policy LLM, $\pi_{ref}$ is the reference LLM, $r_{\phi}$ and $\mathbb{D}_{\text{KL}}$ denotes the reward function and KL divergence respectively.  To ensure proper credit assignment, we implement token masking ($o_t$). The loss is computed exclusively on the agent's outputs ($\tau_t$ and $a_t$), thereby excluding environmental feedback from gradient updates.

%The input $x$ is sampled from dataset $D$, and $\mathcal{H}$ is the output trajectory, interleaved with tool-call feedback. 
% To ensure proper credit assignment, we implement feedback token masking ($o_t$). The loss is computed exclusively on the agent's outputs ($\tau_t$ and $a_t$), thereby excluding environmental feedback from gradient updates1

\begin{figure*}[t]
    \centering 
    \includegraphics[width=0.98\linewidth]
    {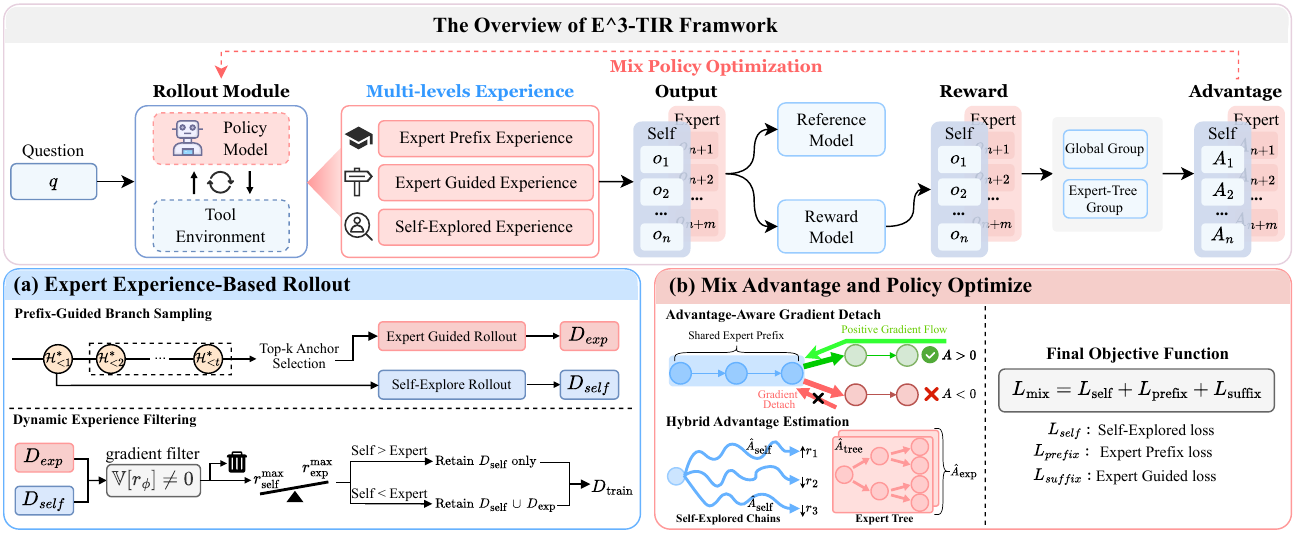}
    \caption{Illustration of our \ourapproach framework. (a) Branching exploration from expert anchors and dynamic experience filtering.  (b) Hybrid advantage estimation and advantage-aware gradient detachment.}
    \label{framwork}
\end{figure*}

\section{Related Work}

\paragraph{Tool Integrated Reasoning.}

Tool-Integrated Reasoning (TIR)~\cite{understandtir, agentrltir} empowers LLMs to overcome the inherent limitations of pure in-context reasoning by leveraging external utilities such as code interpreters and search engines. For question-answer (QA) tasks, efforts such as Search-R1~\cite{search-r1}, R1-Search~\cite{R1-Searcher} , and Re-search~\cite{research} primarily focus on utilizing search tools. Conversely, in mathematical tasks, TIR largely relies on agent-generated Python code to enhance reasoning accuracy and efficiency. Current methodologies for training LLMs in tool invocation present a distinct set of challenges. One path involves Zero-shot RL based online exploration, where models like ToRL~\cite{torl}, OTC~\cite{otc}, and SimpleTIR~\cite{xue2025simpletir} learn autonomous tool calls. However, this often leads to inefficient exploration and redundant tool invocation. The second path utilizes Supervised Fine-Tuning (SFT) for expert knowledge injection, with works like ReTool~\cite{retool}, Tool-star~\cite{toolstar}, focusing on synthesizing high-quality agent data to inject tool-use priors before RL training. 

%Nevertheless, the inherent low-entropy property of SFT typically results in distribution collapse, severely constraining the diversity and ultimate performance gains of subsequent RL optimization. 
%To address this critical diversity-efficiency trade-off, \ourapproach proposes a novel framework that integrates online exploration with expert experience to enable LLMs to acquire tool-calling skills in a more natural and effective manner.

\paragraph{Agent Reinforcement Learning.} 

Reinforcement Learning (RL) plays a pivotal role in enabling LLMs agents to adapt effectively to dynamic and open-ended environments~\cite{survey3}. Recent efforts have primarily focused on addressing the challenges of sparse rewards and complex decision-making in multi-turn interactions~\cite{feng2025group, agentrltir}. RAGEN~\cite{regen} integrates the reasoning state with the environmental interaction into turn-level responses and trains the agent using trajectory-level RL. To further alleviate the issue of reward sparsity,  IGPO~\cite{icpo} introduce more fine-grained rewards, such as leveraging the agent's self-confidence. To overcome the difficulty of capturing complex multi-turn behaviors with trajectory-level RL, Tree-GRPO~\cite{treegrpo} employs a Grouped Relative Policy Optimization based on tree search, allowing the model to effectively learn intermediate decisions. Furthermore, ARPO~\cite{dong2025arpo} secures more diverse reasoning paths through high-entropy branch sampling, and AEPO~\cite{dong2025aepo} advances this direction by proposing a policy optimization method based on entropy balancing.

\section{Methodology}\label{section: Methodology}

In our  \ourapproach framework, we divide the overall training flow into two main parts: 

1) \textit{\textbf{Expert Experience-Based Rollout}}. First, we utilize the current policy to conduct self-exploration and expert-guided exploration, respectively, performing branch sampling from the expert prefixes to form the experience pools $\mathcal{D}_{\text{exp}}$ and $\mathcal{D}_{\text{self}}$ (Section~\ref{text:sample}). Subsequently, we filter the experiences pools based on gradient validity and performance advantage to construct the final training set $\mathcal{D}_{\text{train}}$ (Section~\ref{text:filter}).

2) \textit{\textbf{Mix Advantage and Policy Optimization}}. We categorize advantages into global advantage and expert-tree internal advantage to determine relative trajectory preference, processing different experiences separately (Section~\ref{text:advantage}). To update the policy using mixed experiences, we introduce an off-policy shape and resolve gradient conflicts in shared prefixes via advantage-aware gradient blocking (Section~\ref{text:policy}).

\subsection{Expert Experience-Based Rollout}
\label{text:branch}
\subsubsection{Prefix-Guided Branch Sampling}
\label{text:sample}
    Unlike self-exploration methods~\cite{dong2025arpo,dong2025aepo}, we propose expert-guided branch sampling, using expert prefixes as anchors to reduce search redundancy and ensure reasoning correctness (Theoretical Proof provided in Appendix~\ref{sec:appendix_theory}).

Empirical analysis confirms this advantage: branching from expert prefixes yields significantly higher policy entropy and reward variance compared to self-generated trajectories (Figure~\ref{fig:method}). This demonstrates that expert anchors provide a more informative and gradient-valuable exploration space for subsequent policy optimization. 

Accordingly, for each problem $x$, we allocate a self-explore budget $N$ and an expert-enhanced budget $M$. The sampling process is formalized as three stages:

\paragraph{1. Initialization.} We initialize $\mathcal{D}_{\text{self}}$ with $N$ starting states. For the expert trajectory $\mathcal{H}^{*}$, we measure exploration value using the step-wise policy entropy $h_t$, conditioned on the prefix $\mathcal{H}_{<t}$:
    \begin{equation}
        h_t = - \sum_{y \in Y} \pi_{\theta}(y | \mathcal{H}_{<t}) \log \pi_{\theta}(y |\mathcal{H}_{<t})
    \end{equation}
    where $Y$ denotes the generated token sequence. We construct $\mathcal{D}_{\text{exp}}$ by selecting the top-$k$ high-entropy prefixes, assigning each a unique tree ID $\mathcal{T}$.
    \begin{equation}
        \mathcal{D}_{\text{exp}} \gets \text{Top-K}(\{ (\mathcal{H}^{new}=\mathcal{H}_{<t},\mathcal{T}^{new}) \}_{t=2}^{|\mathcal{H}^{*}|})
    \end{equation}
    
\paragraph{2. Sampling.} To fully leverage expert experience, we sample non-terminated nodes from $\mathcal{D}_{\text{exp}}$ at the end of each round to spawn diverse branches. The sampling probability is determined by the exploration gain $P_t$:
    \begin{equation}
        P_t = \frac{(h_t - h_{t-1})}{\hat{n}_{\mathcal{H}}} \cdot \mathbf{I}(|\mathcal{D}_{\text{exp}}| < M)
    \end{equation}
    where $\alpha$ is a base sampling probability and $\hat{n}_{\mathcal{H}}$ represents the number of branches of this expert tree. If $P_t$ exceeds a threshold $\gamma$ and the budget $M$ is not reached, a copy of the trajectory  is created and inserted into $\mathcal{D}_{\text{exp}}$:
    \begin{equation}
    \mathcal{D}_{\text{exp}} \gets \mathcal{D}_{\text{exp}} \cup \begin{cases}
    \{ \text{Copy}(\mathcal{H}, \mathcal{T}) \}, & \text{if } P_t > \gamma \\
    \emptyset, & \text{otherwise}
    \end{cases}
    \end{equation}
    
\paragraph{3. Expansion.}  For all active states/trajectories $\mathcal{H}$ in $\mathcal{D}_{\text{Self}}$ and $\mathcal{D}_{\text{exp}}$, we execute a single step of outward Rollout (expansion) using $\pi_{\theta}$. The experience pools are updated with the resulting longer trajectories $\mathcal{H}_{\text{new}}$:
    \begin{equation}
        \mathcal{D}_{\text{total}} \gets \text{Update}(\{ \mathcal{H}^{\text{new}}\sim \text{Rollout}(\mathcal{H}) \,\, \text{if} \,\,|\mathcal{H}| < T \})
    \end{equation}
    We iterate Steps 2 and 3 until termination or the turn limit $T$. This constrained rollout ensures parallel efficiency, avoiding the unbounded expansion of MCTS~\cite{stepsearch}.

%Through this sampling, \ourapproach effectively grounds exploration in correct expert prefixes. Furthermore, by maximizing reward variance at uncertain states, it generates step-level contrastive preferences that sharpen critical decision node. The complete procedure is detailed in Algorithm~\ref{alg:sample}.

\subsubsection{Dynamic Experience Filtering}
 \label{text:filter}
To empower the agent to dynamically adjust its knowledge boundaries, we introduce a filtering mechanism based on gradient and performance.

% First, regarding gradient validity, effective policy optimization requires clear reward signals. Consequently, we discard any experience pool lacking reward variance, retaining only subsets $\mathcal{D}'$ that provide valid preference contrast:

First, regarding gradient validity, we always retain $\mathcal{D}_{\text{self}}$ as the global advantage baseline. For the expert pool: if an expert tree $\mathcal{D}_{\text{exp}}^{T^i}$ has zero reward variance $\mathbb{V}$, only one representative trajectory $\mathcal{H}_i$ is kept to reduce redundancy; otherwise, the full tree is retained:

\begin{equation}
\mathcal{D}_{\text{exp}}' = \begin{cases}
\bigcup_{i} \{ \mathcal{H} \sim \mathcal{D}_{\text{exp}}^{T^i} \}, & \text{if } \mathbb{V}[r(\mathcal{D}_{\text{exp}}^{T^i})] = 0 \\
\bigcup_{i} \mathcal{D}_{\text{exp}}^{T^i}, & \text{otherwise}
\end{cases}
\end{equation}

Second, to prevent the model from regressing to suboptimal patterns, we construct the final training set $\mathcal{D}_{train}$ based on a performance criterion. Specifically, we compare the maximum rewards yielded by self-exploration against those from expert guidance:
\begin{equation}
\mathcal{D}_{\text{train}} = 
\begin{cases} 
\mathcal{D}_{\text{self}}, & \text{if } r^{max}_{\phi}(\mathcal{D}_{\text{self}}) \geq r^{max}_{\phi}(\mathcal{D}'_{\text{exp}}) \\
\mathcal{D}_{\text{self}} \cup \mathcal{D}'_{\text{exp}}, & \text{otherwise}
\end{cases}
\end{equation}

This logic enforces an adaptive learning curriculum: expert trajectories are discarded when the agent matches or exceeds their performance, prioritizing the agent's superior policy. 

\begin{figure}[t]
	\centering
	\includegraphics[width=0.95\columnwidth]{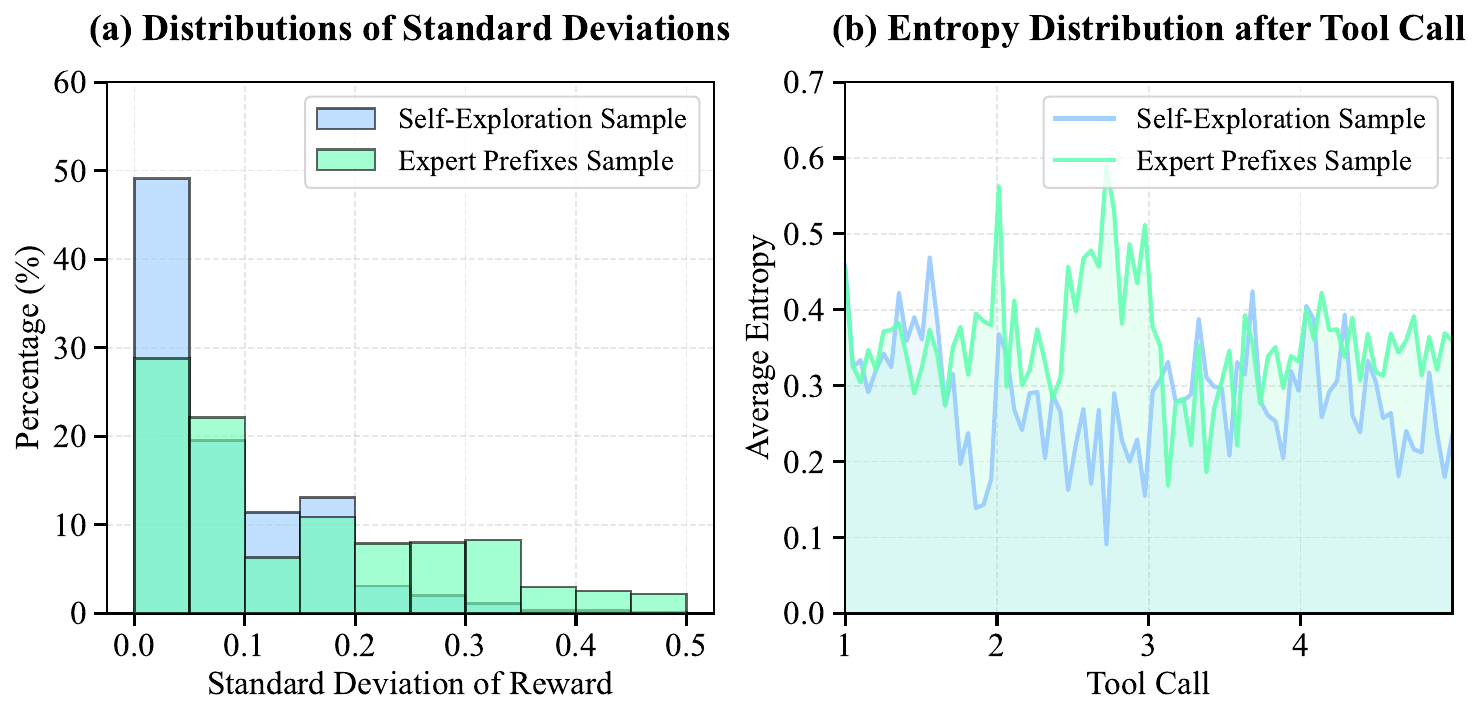}
    	\caption{Distribution of Standard Deviation and Entropy under Different sampling strategies.}
        \label{fig:method}
\end{figure}

\subsection{Mix Advantage and Policy Optimize}
\subsubsection{Hybrid Advantage Estimation}
\label{text:advantage}

Upon collecting the final experience pool $\mathcal{D}_{\text{train}}$, the primary step is to determine the relative advantage for each trajectory $\mathcal{H}$ in $\mathcal{D}_{\text{train}}$. This differentiation of relative advantage occurs at two critical levels:

\paragraph{1) Global Advantage} ($\hat{A}_{\text{global}}$). This metric measures the trajectory's performance relative to the average performance across the entire pool $\mathcal{D}_{\text{train}}$. We leverage the global advantage for baseline estimation, which is crucial for stabilizing training:
\begin{equation}
    \hat{A}_{\text{global}}(\mathcal{H}^i)=\frac{r_{\phi}(\mathcal{H}^i)-\text{mean}(r_{\phi}(\mathcal{H}^{j}))^{D_{\text{train}}}_{j=1}}{\text{mean}(r_{\phi}(\mathcal{H}^{j}))^{D_{\text{train}}}_{j=1}}
\end{equation}

\paragraph{2) Expert-Tree Relative Advantage}  ($\hat{A}_{\text{exp-tree}}$). We compute an internal advantage within $\mathcal{D}_{\text{exp}}$ to measure variance among branches sharing the prefix $\mathcal{T}$. This guides the model to capture the expert's granular thought patterns.

\begin{equation}
    \hat{A}_{\text{tree}}(\mathcal{H}^i)=  \frac{r_{\phi}(\mathcal{H}^i)-\text{mean}(r_{\phi}(\mathcal{H}^{j}))^{D^{(\mathcal{T}=\mathcal{T}^i)}_{\text{exp}}}_{j=1}}{
    |D^{(\mathcal{T}=\mathcal{T}^i)}_{\text{exp}}| \cdot \text{mean}(r_{\phi}(\mathcal{H}^{j}))^{D^{(\mathcal{T}=\mathcal{T}^i)}_{\text{exp}}}_{j=1}}
\end{equation}
Since different expert trees may share the same prefix, we normalize the advantage by dividing it by the number of trajectories within the current expert tree to prevent over-optimization.

Consequently, the advantage estimation for trajectories within $\mathcal{D}_{\text{self}}$ is solely derived from the $\hat{A}_{\text{global}}(\cdot)$. Conversely, for trajectories in $\mathcal{D}_{\text{exp}}$, the advantage estimation combines both the  $\hat{A}_{\text{global}}(\cdot)$ and $\hat{A}_{\text{tree}}(\cdot)$.

\begin{equation}
\label{eq:advantage_definitions}
    \left\{
    \begin{aligned}
        \hat{A}_{\text{self}}(\mathcal{H}_{i}) &= \hat{A}_{\text{global}}(\mathcal{H}_i) \\
        \hat{A}_{\text{exp}}(\mathcal{H}_{i}) &= \hat{A}_{\text{global}}(\mathcal{H}_i) + \hat{A}_{\text{tree}}(\mathcal{H}_i)
    \end{aligned}
    \right.
\end{equation}

\subsubsection{Mix Policy Optimization}
\label{text:policy}

Since distribution shifts in expert trajectories cause training instability, we adopt off-policy policy shape~\cite{luffy} to effectively utilize the mixed experience $\mathcal{D}_{\text{train}}$. The policy object is:

\begin{equation}
\begin{split}
\mathcal{J}_{\text{Hybrid}}&(\theta) = \frac{1}{Z} \Bigg( \underbrace{\sum_{\mathcal{H}_i \in \mathcal{D}_{\text{self}}} \sum_{t=1}^{|\mathcal{H}_i|} \text{CLIP}(r_{i,t}(\theta), \hat{A}^{\text{self}}_i, \epsilon)}_{\text{Pure Self-Exploration (On-Policy)}}  \\
& + \quad \underbrace{\sum_{\mathcal{H}_k \in \mathcal{D}_{\text{exp}}} \left( \mathcal{L}_{\text{prefix}}(\mathcal{H}_k) + \mathcal{L}_{\text{suffix}}(\mathcal{H}_k) \right)}_{\text{Expert prefixes and Expert guidance  (Mix-Policy)}} \Bigg)
\end{split}
\end{equation}

The mixed objective is divided into two parts: (1) the self-exploration part and (2) the mixed policy part. The mixed policy component, in turn, includes the prefix and suffix parts.  Their specific forms are as follows:

\begin{subequations}
\begin{align}
\mathcal{L}_{\text{prefix}}(\mathcal{H}_k) &= \sum_{t=1}^{L_{\text{prefix}}} \text{CLIP}(\rho_{k,t}(\theta), \hat{A}^{\text{exp}}_k)  \\
\mathcal{L}_{\text{suffix}}(\mathcal{H}_k) &= \sum_{t=L_{\text{prefix}}+1}^{|\mathcal{H}_k|} \text{CLIP}(r_{k,t}(\theta), \hat{A}^{\text{exp}}_k) \label{eq:loss_suffix}
\end{align}
\end{subequations}

Where the importance sampling ratio $r_{k,t}$ is defined at the token level $t$, and $\rho_{k,t}$ is the policy reshaping function $f(\pi_{\theta})=\pi/(\pi+\lambda)$

Branching sampling creates optimization conflicts where failed explorations ($A<0$) incorrectly penalize valid shared prefixes. We address this via \textit{Advantage-Aware Gradient Detaching}, which selectively blocks prefix gradients for failure paths ($A<0$) while retaining full updates for successful ones ($A>0$). The modified objective is:

% Within the \ourapproach framework, branching sampling creates a gradient conflict on shared prefixes: Failed explorations ($A<0$) incorrectly penalize the valuable prior knowledge in the shared expert prefix.

% To address this, we propose advantage-aware gradient detaching, which implements an asymmetric optimization by gating the gradient based on the advantage $A_t$: \textbf{(1) Success} ($A_i > 0$): The full gradient reinforces the entire path. \textbf{(2) Failure} ($A_i < 0$): The gradient to the shared prefix is detached (blocked), penalizing only the model's incorrect suffix. The modified objective $\mathcal{L}_{\text{prefix}}$ is as follows:
\begin{equation}
\mathcal{L}_{\text{prefix}}(\mathcal{H}_k) = \sum_{t=1}^{L_{pre}} \text{CLIP}(\rho_{k,t}(\theta), \hat{A}^{\text{exp}}_k) \cdot \mathbb{I}(\hat{A}^{\text{exp}}_k > 0) \label{eq:loss_prefix} 
\end{equation}
Through this asymmetric optimization strategy, we theoretically eliminate the optimization ambiguity of the shared prefix under different branches.

\section{Experiments}
\label{exp}

\input{tables/compare_to_algothrim}

\begin{figure*}[htbp]
	\centering
	\begin{subfigure}[b]{0.495\linewidth}
    		\centering
    		\includegraphics[width=\linewidth]{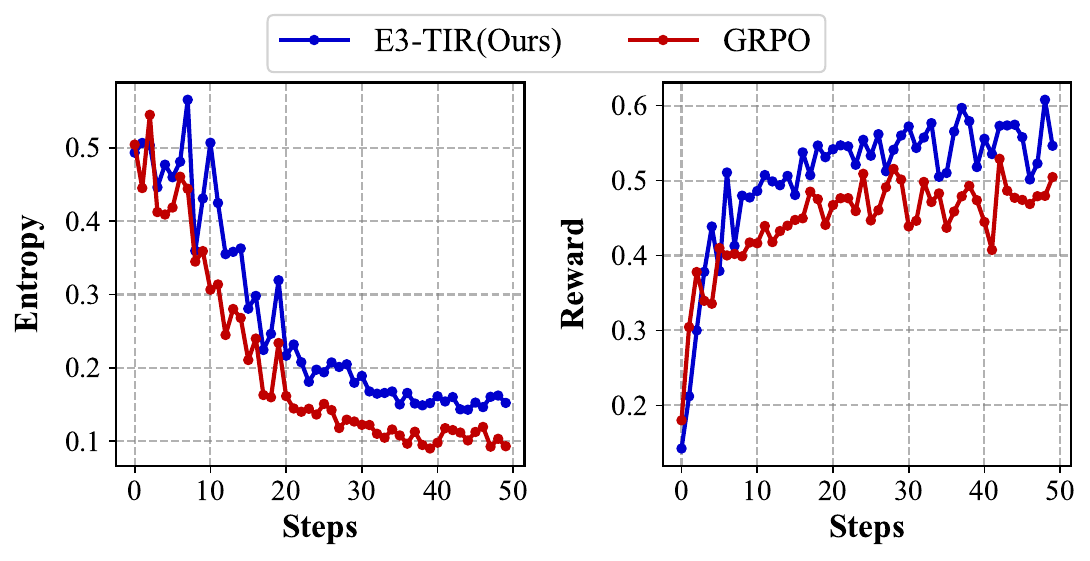}
    		\caption{Entropy and Reward Curves on Warm-Up Stage}
    		\label{fig3:train_1}%文中引用该图片代号
	\end{subfigure}
	\hfill % 添加水平填充
	\begin{subfigure}[b]{0.495\linewidth}
    		\centering
    		\includegraphics[width=\linewidth]{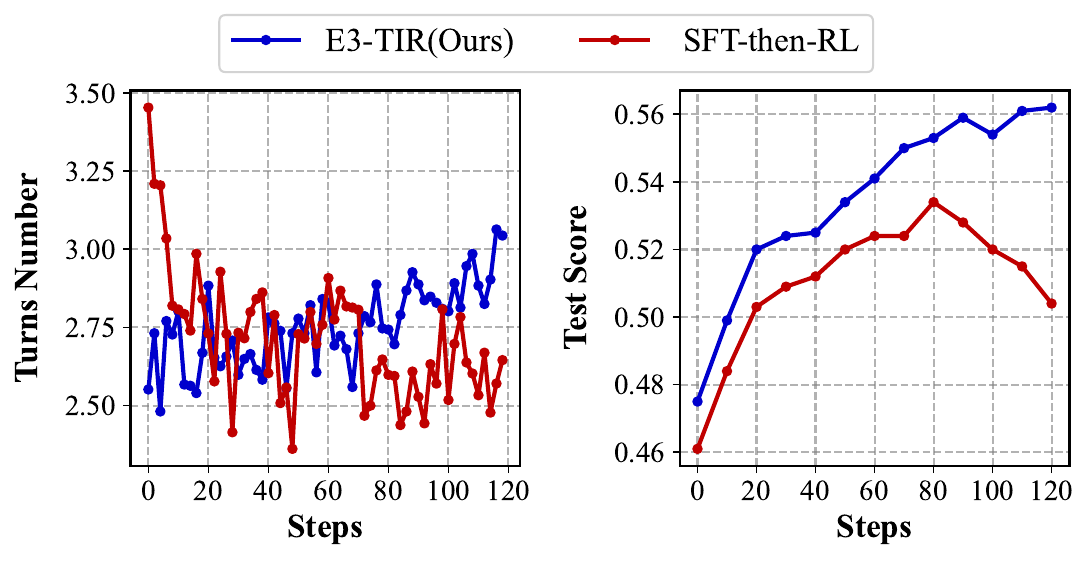}
    		\caption{Turns Number and Test Score on Post-RL Stage}
    		\label{fig3:train_2}%文中引用该图片代号
	\end{subfigure}
	\caption{ Comparison of the training curves on Qwen2.5-3B-Instruct.  }
	\label{fig3:train}
\end{figure*}

\subsection{Experiment Setup}
\label{setup}

\paragraph{Tasks and Metrics.}
To comprehensively evaluate the effectiveness of our \ourapproach algorithm in TIR, we conduct experiments on the following two types of tool-use tasks:

1) \textit{\textbf{Mathematical Reasoning:}}. In this task, the model can automatically generate and execute code. We use metrics such as AIME24, AIME25, MATH500~\cite{math500}, AMC23, and GSM8K~\cite{cobbe2021gsm8k} to evaluate the model's mathematical reasoning capabilities.

2) \textit{\textbf{Knowledge-Intensive Reasoning:}}  We evaluate on knowledge question benchmarks, including:  HotpotQA~\cite{yang2018hotpotqa}, 2WikiMultiHopQA~\cite{xanh2020_2wikimultihop}, Musique~\cite{trivedi2021musique}, and Bamboogle~\cite{Bamboogle}, SimpleQA~\cite{Bamboogle}.

\paragraph{Baselines.} We compare the proposed \ourapproach against the following baseline: (i) \textbf{\textit{Training Pipeline.}}We compare against SFT-only, SFT-then-RL (via GRPO), Zero-RL (via GRPO),  trained on the same dataset. (ii) \textbf{\textit{LLM-based Search Agent.}} including Search-o1~\cite{Search-o1}, Search-R1~\cite{jin2025search} and Tree-GRPO~\cite{treegrpo}. (iii) \textbf{\textit{LLM-based Math Agent.}} including ToRL~\cite{torl} and SimpleTIR~\cite{xue2025simpletir}.  (iv) \textbf{\textit{Mult-Tool Integrated Agent.}} including Tool-Star~\cite{toolstar}, ReCall~\cite{chen2025researchlearningreasonsearch} and ARPO~\cite{dong2025arpo}.  

\paragraph{Training Guideline.} To ensure reproducibility, all training frameworks and datasets are sourced from open-access resources.  Specifically, the expert trajectories data are from the Tool-Star SFT dataset. For the RL data, we also utilize the Tool-Star RL dataset.

\textbf{1. Initialization:} To mitigate reward collapse during the initial RL training stage, the model is initialized using only a small amount of self-distillation data.. 

\textbf{2. RL Stage:} The RL phase adopts a two-stage training paradigm:
\begin{itemize}[noitemsep,nolistsep]
    \item \textbf{E3-TIR (Warm Up)}. In this stage, the model undergoes mix policy training using the SFT dataset augmented with expert experience.
    \item \textbf{E3-TIR (Post-RL)}. In this stage, the model is trained using the standard RL dataset without any specific modifications.
\end{itemize} 

The detailed introduction of training setup are listed in Appendix~\ref{Appendix:training}. 

\input{tables/main_table_compare_method}

\subsection{Main Results}

\paragraph{Comparison with Training Paradigms.} Table~\ref{tab:compare_modified} presents our main experimental results. Compared to baseline methods, \ourapproach consistently achieves superior performance. Whether applied to smaller 3B models or larger 7B/8B models, our mechanism yields a 6\% average performance gain while simultaneously achieving higher tool-use efficiency.

\paragraph{Comparison with SOTA Methods.}  Table~\ref{tab:main_table} benchmarks \ourapproach against current open-source SOTA methods. Results demonstrate that, under comparable parameter scales, our approach exhibits superior performance for both single-tool and multi-tool agents, maintaining a robust competitive edge across diverse benchmarks.

\begin{figure}[t]
	\centering
	\includegraphics[width=1\columnwidth]{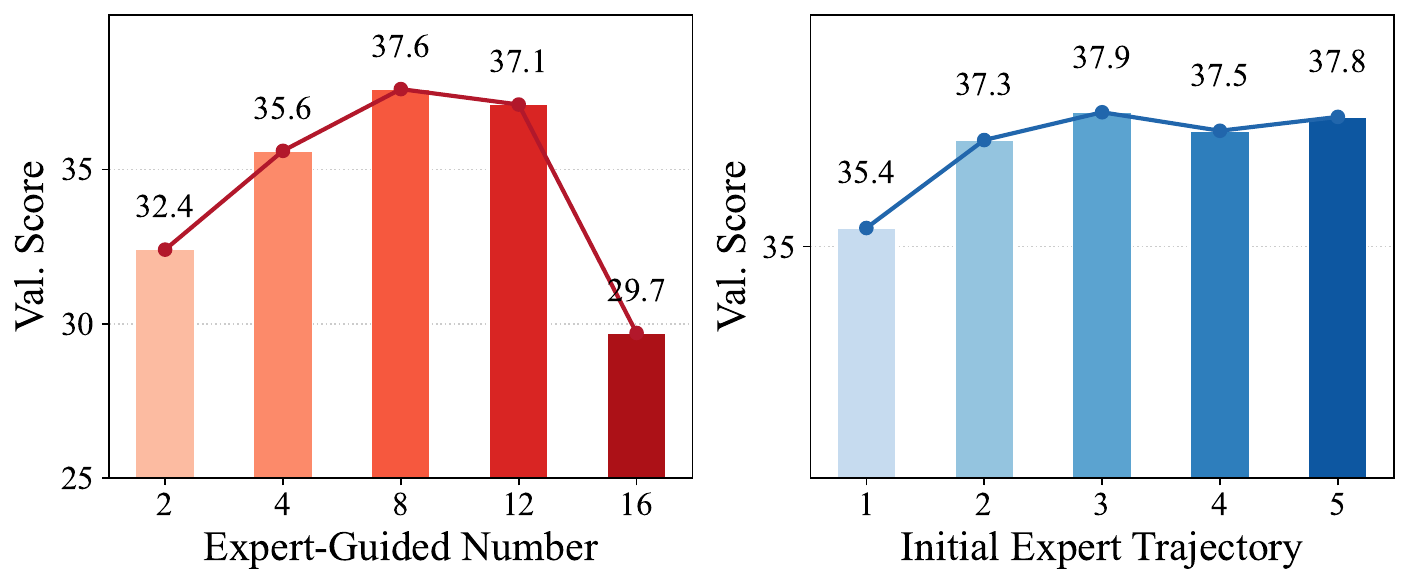}
    	\caption{Ablation analysis of different hyper-parameters in Expert Branch Sample.}
        \label{fig:ablation_sample}
\end{figure}

\paragraph{Balancing Exploration Quality and Efficiency.} As shown in Figure~\ref{fig3:train_1}, \ourapproach achieves rapid reward ascent by leveraging expert anchors, in contrast to the cold start of Zero-RL. Simultaneously, the model maintains high policy entropy, indicating that it effectively exploits expert knowledge without loss diversity.

\paragraph{Good Environmental Adaptability.} As shown in Figure~\ref{fig3:train_2}, the SFT model struggles to adapt to the multi-turn interactions, leading to declining tool usage and an overfitting ``fast rise, then drop’’ pattern in accuracy Conversely, \ourapproach demonstrates robust exploration with steadily increasing tool usage and maintains a stable upward performance curve throughout the process.

\begin{figure}[t]
	\centering
	\includegraphics[width=1\columnwidth]{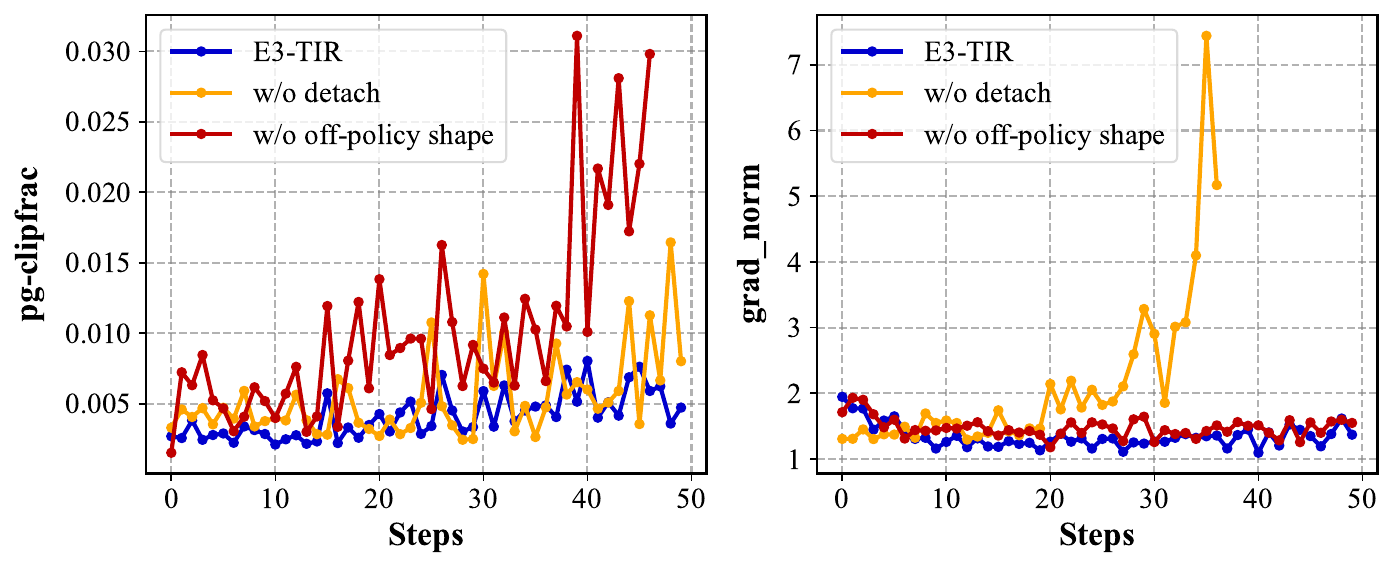}
    	\caption{Ablation analysis of clip ratio and gradient norm on different mix policy optimization.}
        \label{fig:ablation_mix}
\end{figure}

\begin{table}[t]
    \centering
    % 设置字体大小
    \fontsize{9pt}{12pt}\selectfont
    % 需要导入宏包: \usepackage{multirow}, \usepackage{booktabs}, \usepackage{colortbl}, \usepackage{xcolor}
    % 使用 resizebox 将表格宽度限制为栏宽
    \resizebox{\columnwidth}{!}{
        \begin{tabular}{lcccc}
            \toprule
            % 表头第一行：Method 跨两行，Knowledge 和 Computational 各跨两列
            \multirow{2}{*}{\textbf{Method}} & \multicolumn{2}{c}{\textbf{Knowledge.}} & \multicolumn{2}{c}{\textbf{Computational.}} \\
            % 部分横线，分别在 2-3 列和 4-5 列下方
            \cmidrule(lr){2-3} \cmidrule(lr){4-5}
            % 表头第二行
             & \textbf{HotQA} & \textbf{Musique.} & \textbf{ACM23} & \textbf{MATH} \\ 
            \midrule
            % 第一行数据：Tool-Star (3B)，带背景色
            \rowcolor[RGB]{236,244,252} 
            \ourapproach (Warm up) & 44.9 & 19.9 & 43.3 & 65.0 \\
            \midrule
            % 消融实验数据行
            % 使用 _{} 实现下标，\textcolor{red}{} 实现红色字体
            \textit{w/o} $\hat{A}_{\text{global}}$ & $35.4_{\textcolor{red}{(-9.5)}}$ & $13.7_{\textcolor{red}{(-6.2)}}$ & $36.8_{\textcolor{red}{(-6.5)}}$ & $56.3_{\textcolor{red}{(-8.7)}}$ \\
            \textit{w/o} $\hat{A}_{\text{tree}}$ & $42.6_{\textcolor{red}{(-2.3)}}$ & $17.8_{\textcolor{red}{(-2.1)}}$ & $43.5_{\textcolor{red}{(+0.2)}}$ & $63.1_{\textcolor{red}{(-1.9)}}$ \\
            \bottomrule
        \end{tabular}
    }
    \caption{Ablation on Hybrid Advantages Estimation.}
    \label{table:ablation}
\end{table}

\subsection{Ablation Study}\label{txt: Comparison with others}

\paragraph{Component-Level Ablation.} To isolate the contribution of each core design choice, we systematically remove individual components while keeping the same expert prefixes and training budget. We additionally compare against a simple weighted SFT+RL mixture baseline (Mix SFT+RL) that replaces our loss with a weighted combination following~\citet{hpt}. Results are reported over 3 independent runs.

\begin{table}[t]
    \centering
    \fontsize{9pt}{12pt}\selectfont
    \renewcommand{\arraystretch}{1.1}
    \resizebox{\columnwidth}{!}{
        \begin{tabular}{lcccc}
            \toprule
            \multirow{2}{*}{\textbf{Method}} & \multicolumn{2}{c}{\textbf{Knowledge.}} & \multicolumn{2}{c}{\textbf{Computational.}} \\
            \cmidrule(lr){2-3} \cmidrule(lr){4-5}
             & \textbf{HotQA} & \textbf{Musique.} & \textbf{AMC23} & \textbf{MATH} \\
            \midrule
            \rowcolor[RGB]{236,244,252} 
            \ourapproach & \textbf{51.4} & \textbf{26.9} & \textbf{56.0} & \textbf{85.2} \\
            \midrule
            \textit{w/o} Branch & $47.2_{\textcolor{red}{(-4.2)}}$ & $23.1_{\textcolor{red}{(-3.8)}}$ & $52.3_{\textcolor{red}{(-3.7)}}$ & $83.6_{\textcolor{red}{(-1.6)}}$ \\
            \textit{w/o} AAGD & $15.3_{\textcolor{red}{(-36.1)}}$ & $8.2_{\textcolor{red}{(-18.7)}}$ & $24.8_{\textcolor{red}{(-31.2)}}$ & $48.6_{\textcolor{red}{(-36.6)}}$ \\
            \textit{w/o} Hybrid Adv. & $49.5_{\textcolor{red}{(-1.9)}}$ & $25.6_{\textcolor{red}{(-1.3)}}$ & $55.2_{\textcolor{red}{(-0.8)}}$ & $84.1_{\textcolor{red}{(-1.1)}}$ \\
            Mix SFT+RL & $22.4_{\textcolor{red}{(-29.0)}}$ & $11.3_{\textcolor{red}{(-15.6)}}$ & $32.5_{\textcolor{red}{(-23.5)}}$ & $56.8_{\textcolor{red}{(-28.4)}}$ \\
            \bottomrule
        \end{tabular}
    }
    \caption{Component ablation on Qwen2.5-3B-Instruct. \textit{w/o} Branch: remove expert-guided branching; \textit{w/o} AAGD: remove advantage-aware gradient detaching; \textit{w/o} Hybrid Adv.: remove expert-tree advantage; Mix SFT+RL: weighted SFT+RL~\cite{hpt}.}
    \label{table:component_ablation}
\end{table}

As shown in Table~\ref{table:component_ablation}, removing AAGD leads to training collapse (accuracy drops to 20.6 / 14.7), confirming it is essential for resolving gradient conflicts on shared prefixes. Branching exploration contributes the largest performance gain among the remaining components: its removal causes a greater accuracy drop than removing hybrid advantage estimation, highlighting that diverse expert-anchored exploration is the primary driver of \ourapproach's effectiveness. The Mix SFT+RL baseline also collapses (27.6 / 32.3), demonstrating that standard SFT loss is ineffective for multi-turn, long-horizon tool-use tasks due to severe distribution shifts.

\paragraph{Parameter Ablation for Expert Sample.} We conducted an ablation study on two core parameters within the expert-guided sample module:

1) \textit{Balancing Self-Exploration ($N$) and Expert-Guided ($M$)}.  We fixed the total sampling budget at 16 and varied $M$ to evaluate model performance after the warm-up stage. As shown in Figure~\ref{fig:ablation_sample} (left), performance improves as $M$ increases (for $M < 8$), demonstrating that expert-guided trajectories effectively accelerate early convergence. However, performance deteriorates as $M$ increases further, indicating that retaining a sufficient proportion of self-generated trajectories is critical for stabilizing training.

2) \textit{Initial Expert Trajectories ($K$)}. With $M$ fixed at 8, we varied the initial value of $K$. As illustrated in Figure~\ref{fig:ablation_sample} (right), performance peaks at $K=3$. Beyond this point, increasing $K$ yields no significant performance variation. This is primarily attributed to the fact that the average number of interaction turns in expert trajectories typically ranges between 2 and 3.

\begin{figure}[t]
	\centering
	\includegraphics[width=1\columnwidth]{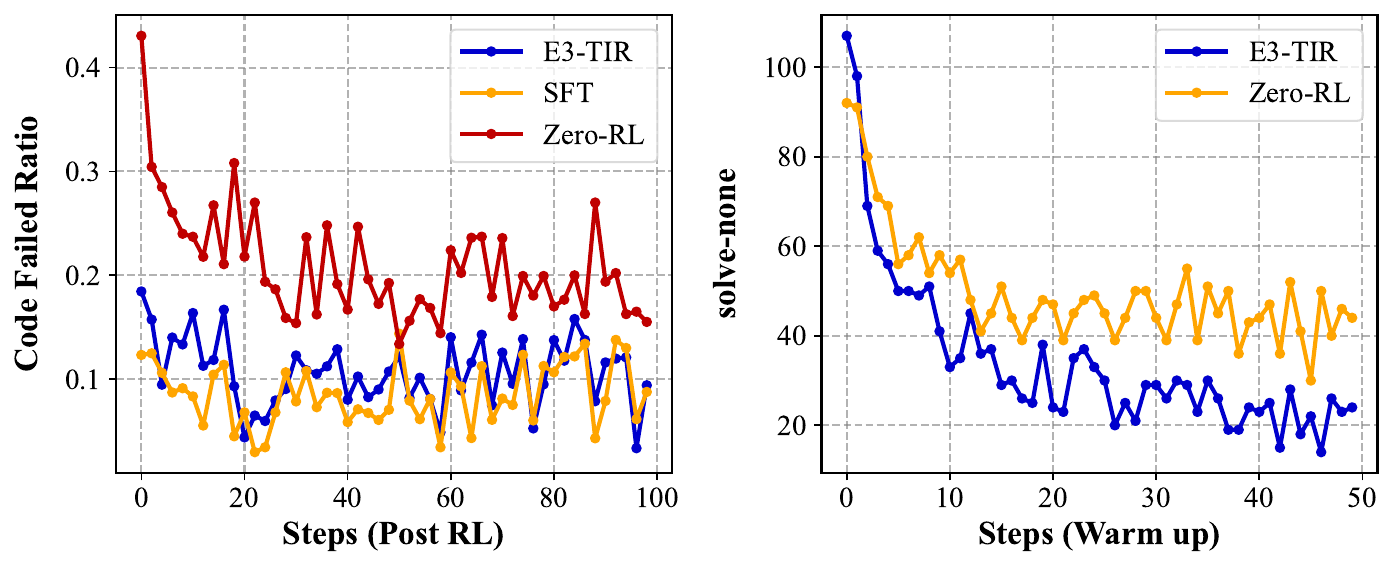}
    	\caption{Comparison of Code Failed Rate and Solve-None curves across different training stages.}
        \label{fig:code_none}
\end{figure}

\paragraph{Effectiveness of Mix Policy Optimization.} We validate the contribution of each component through ablation studies:

1) \textit{Mix Advantage Estimation}: Table~\ref{table:ablation} shows that $A_{tree}$ alone causes training collapse, while the combined $A_{tree} + A_{global}$ significantly outperforms $A_{global}$ alone. This confirms that internal advantage captures fine-grained step differences, effectively enhancing sample efficiency.

2) \textit{Off-policy Reshaping}: As shown in Figure~\ref{fig:ablation_mix} (left), direct updates with mixed data cause clip ratio spikes due to distributional disparity. In contrast, off-policy reshaping corrects these shifts, maintaining training stability.

3) \textit{Shared Prefix Gradient Detach}: Figure~\ref{fig:ablation_mix} (right) demonstrates that without detachment, gradient norms oscillate due to optimization conflicts. AAGD ensures stable gradient flow, preserving the expert prefix by resolving these conflicts.

\subsection{Analysis}

\paragraph{Code Execution Failure Analysis.} As shown in Figure~\ref{fig:code_none} (left), we compared code failure rates during the Post-RL stage. Benefiting from the warm-up, \ourapproach exhibits a significantly lower initial error rate than Zero-RL. Notably, \ourapproach consistently maintains a low level, demonstrating superior robustness in tool usage. 
%while the SFT model's failure rate gradually rises in later stages, 

\begin{table}[t]
    \centering
    \fontsize{9pt}{12pt}\selectfont
    \renewcommand{\arraystretch}{1.1}
    \resizebox{\columnwidth}{!}{
        \begin{tabular}{lcccc}
            \toprule
            \textbf{Method} & \textbf{Avg Score} & \textbf{Avg Calls} & \textbf{Fail Rate} & \textbf{Redundancy} \\
            \midrule
            Only SFT & 32.9 & 2.24 & 4.5\% & 0.6\% \\
            SFT-then-RL & 44.2 & \textbf{1.82} & \textbf{3.8\%} & 0.8\% \\
            Zero-RL & 43.2 & 2.52 & 7.4\% & 0.5\% \\
            \midrule
            \rowcolor[RGB]{236,244,252}
            \ourapproach & \textbf{46.7} & 1.97 & 4.0\% & \textbf{0.6\%} \\
            \bottomrule
        \end{tabular}
    }
    \caption{Tool-calling behavior audit on Qwen2.5-3B-Instruct. Avg Calls: average tool invocations per query; Fail Rate: proportion of invalid tool calls; Redundancy: ratio of repeated calls.}
    \label{table:tool_audit}
\end{table}

\paragraph{Tool-Calling Behavior Audit.} Beyond task accuracy, we audit tool-calling efficiency across training paradigms. As shown in Table~\ref{table:tool_audit}, \ourapproach achieves an average of 1.97 tool calls per query, substantially fewer than Zero-RL's 2.52, indicating that expert anchors help establish efficient invocation patterns during the warm-up phase. The tool failure rate of \ourapproach (4.0\%) is nearly half that of Zero-RL (7.4\%), confirming that expert-guided exploration effectively reduces invalid syntax errors and ill-formed queries. Redundancy rates remain consistently low across all methods, suggesting that repeated tool calls are primarily driven by unresolvable code errors rather than training paradigm differences.

\paragraph{Extension of Capability Boundary.} We tracked the proportion of effectively unsolvable problems during training. As illustrated in Figure~\ref{fig:code_none} (right), compared to Zero-RL, \ourapproach demonstrates a continuous decline in the Solve-None rate, stabilizing at a lower final level. This confirms that expert experience effectively expands the model's knowledge boundary, enabling it to resolve problems previously beyond its reach.

\paragraph{Warm-up Stage Cost and Efficiency Analysis.} We evaluated training efficiency gains across different warm-up settings using the ROI metric (Appendix~\ref{app:cost}). As shown in Figure~\ref{fig:heatmap}a, ROI exhibits an ``inverted-U’’ trend with increasing warm-up steps, attributed to the model's diminishing reliance on experts. Meanwhile, gains from expert data size saturate after 4k samples, demonstrating that minimal expert data is sufficient for efficient startup.

\paragraph{Comparison with Offline Fusion Baselines.} We further benchmark our mixed strategy against prior offline fusion methods~\cite{luffy, hpt}, as illustrated in Figure~\ref{fig:heatmap}b. While \ourapproach outperforms the off-policy free baseline by 19\%, alternative methods like Luffy and HPT (utilizing SFT loss) result in decreased performance. This advantage stems from the task's complexity: robust multi-turn TIR is difficult to achieve through sequence-level optimization alone.
\begin{figure}[t]
	\centering
	\includegraphics[width=0.95\columnwidth]{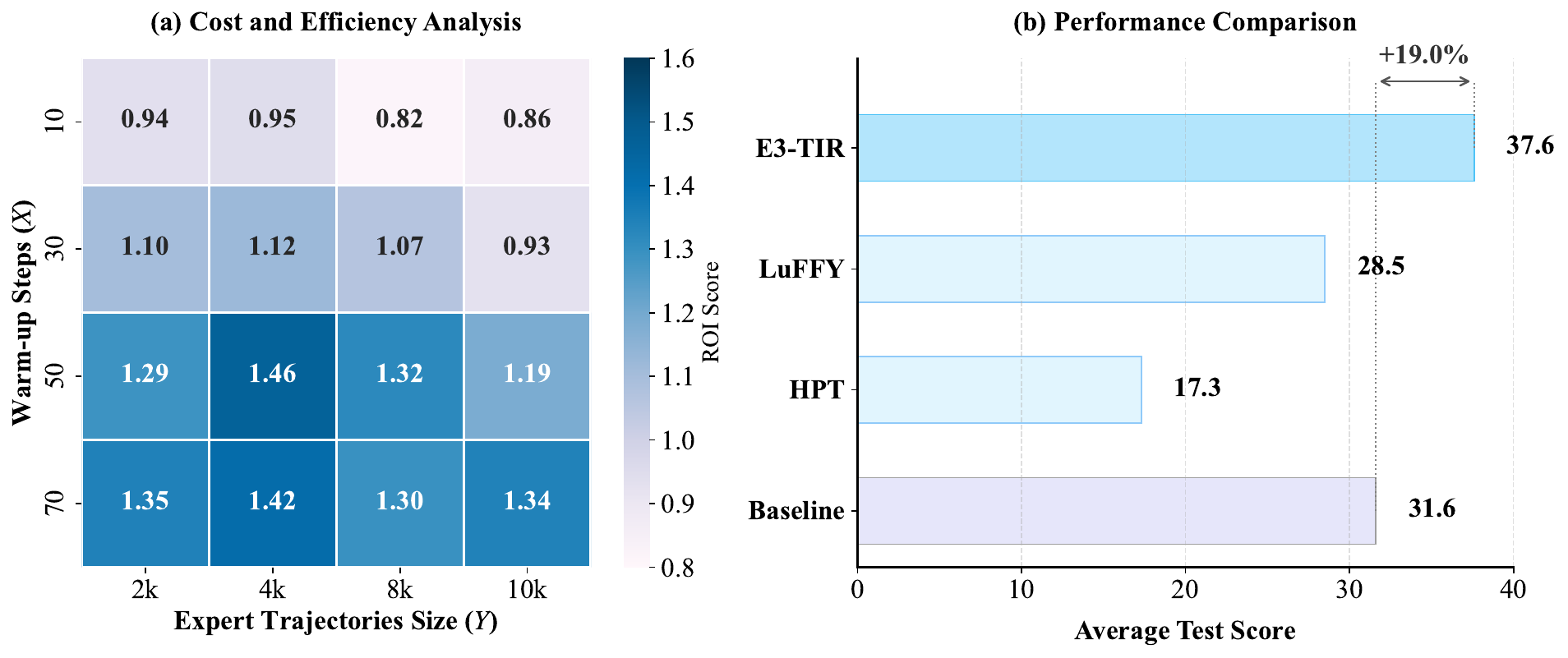}
    	\caption{(a) Cost-Benefit Heatmap during the Warm-up Stage. (b) Comparison with Other Mix Strategies.}
        \label{fig:heatmap}
\end{figure}
\begin{table}[t]
    \centering
    \fontsize{9pt}{12pt}\selectfont
    \renewcommand{\arraystretch}{1.1}
    \resizebox{\columnwidth}{!}{
        \begin{tabular}{lccc}
            \toprule
            \textbf{Method} & \textbf{ALFWorld} & \textbf{SCIWorld} & \textbf{Avg.} \\
            \midrule
            Only SFT & 62.5 & 43.0 & 52.8 \\
            SFT-then-RL & 81.5 & 65.0 & 73.3 \\
            Zero-RL & \underline{86.5} & \underline{63.5} & \underline{75.0} \\
            \midrule
            \rowcolor[RGB]{236,244,252}
            \ourapproach & \textbf{89.0} & \textbf{67.5} & \textbf{78.3} \\
            \bottomrule
        \end{tabular}
    }
    \caption{Generalization on long-horizon agent tasks using Qwen2.5-3B-Instruct. Expert data: 2K trajectories from AgentGYM~\cite{agentgym} per task.}
    \label{table:long_horizon}
\end{table}

\paragraph{Generalization on Long-Horizon Tasks.} To evaluate adaptability beyond math and QA, we conduct experiments on ALFWorld~\cite{alfworld} and SCIWorld~\cite{sciworld}, two long-horizon agent exploration tasks requiring up to 30 interaction turns. As shown in Table~\ref{table:long_horizon}, \ourapproach achieves the best results on both tasks (89.0 and 67.5). Notably, Zero-RL (86.5) outperforms SFT-then-RL (81.5) on ALFWorld, suggesting that SFT-induced distribution shifts can be particularly harmful in stateful environments. \ourapproach further improves upon Zero-RL by effectively leveraging expert guidance without sacrificing exploration flexibility, demonstrating its generalizability to complex, multi-step agent settings.

\section{Conclusion}
In this paper, we introduce \ourapproach , a framework that optimizes agent training by dynamically integrating different experience. By utilizing branching exploration around expert ``anchors’’ alongside a mixed policy optimization mechanism , \ourapproach balances exploration diversity with efficiency, transforming experience utilization into proactive acquisition and establishing a new paradigm for LLMs agent training.

\section*{Acknowledgements}
This work was supported in part by National Natural Science Foundation of China (62476070), Shenzhen Science and Technology Program \seqsplit{(JCYJ20241202123503005, \, GXWD20231128103232001, \,ZDSYS20230626091203008,\, KQTD20240729102154066)}, Department of Science and Technology of Guangdong (2024A1515011540) and National Key R\&D Program of China (SQ2024YFE0200592).

\section*{Limitations}
In this section, we discuss the limitations identified in this study that point towards directions for future improvement. First, while \ourapproach significantly mitigates the dependency on massive synthetic data, it remains sensitive to the quality of initial expert prefixes. Since the core mechanism relies on branching around these ``anchors’’, suboptimal or low-diversity seeds may inadvertently bias the search space, thereby constraining the model's exploration efficiency in highly complex tasks. Furthermore, our current evaluation is primarily strictly within the domains of math and QA. Consequently, the framework's generalizability to strictly more sophisticated, real-world scenarios—particularly those necessitating the simultaneous orchestration of multiple tools or maintaining state consistency over long-horizon interactions—remains to be fully validated.
% \paragraph{Dependency on Expert Trajectories.} Although \ourapproach substantially reduces the requirement for synthetic data to only 10\% of standard baseline volumes, its core mechanism remains centered on performing branching exploration around initial expert "anchors". The quality and diversity of these expert seeds may impose a performance ceiling; if initial trajectories are suboptimal or cover limited reasoning paths, the model's exploration efficiency in extremely complex tasks might be constrained.

% \paragraph{Scalability of Tool Categories.} Our experimental evaluation primarily focuses on two representative TIR tasks: mathematical reasoning via code interpreters and knowledge-intensive QA via search engines. While \ourapproach demonstrates superior robustness in these domains, its adaptability to more sophisticated real-world scenarios—such as multi-tool orchestration, long-horizon interactions with memory constraints, or dynamically evolving API environments—remains to be fully validated.

% \paragraph{Sensitivity to Hyper-parameters.} Ablation studies indicate that performance is sensitive to key hyper-parameters, specifically the expert guidance budget ($M$) and the number of initial expert trajectories ($K$). We observed that an excessively high $M$ can disrupt the stability of self-generated exploration, leading to diminished performance gains. Consequently, deploying \ourapproach across different base model architectures or novel task domains may necessitate additional hyper-parameter optimization to achieve peak efficiency.

\nocite{langley00}

\bibliography{main}
%\bibliographystyle{acl2026}

%%%%%%%%%%%%%%%%%%%%%%%%%%%%%%%%%%%%%%%%%%%%%%%%%%%%%%%%%%%%%%%%%%%%%%%%%%%%%%%
%%%%%%%%%%%%%%%%%%%%%%%%%%%%%%%%%%%%%%%%%%%%%%%%%%%%%%%%%%%%%%%%%%%%%%%%%%%%%%%
% APPENDIX
%%%%%%%%%%%%%%%%%%%%%%%%%%%%%%%%%%%%%%%%%%%%%%%%%%%%%%%%%%%%%%%%%%%%%%%%%%%%%%%
%%%%%%%%%%%%%%%%%%%%%%%%%%%%%%%%%%%%%%%%%%%%%%%%%%%%%%%%%%%%%%%%%%%%%%%%%%%%%%%
\newpage

\appendix

%\onecolumn

\startcontents[appendix]
% 设置局部目录显示到 subsection (2) 或 section (1) 级别
% \setcounter{tocdepth}{2} 
% \printcontents[appendix]{ }{0}{\section*{Appendix}}

% 设置局部目录显示到 subsection (2) 或 section (1) 级别
{
\hypersetup{linkcolor=black}
\printcontents[appendix]{}{0}{\section*
{Appendix}}
}

\newpage
\section{The Use of Large Language Models}
Throughout the preparation of this manuscript, large language models were employed exclusively for light stylistic refinement and the occasional grammatical adjustment. Every conceptual insight, analytical thread, and interpretive conclusion emerged from the authors themselves; no algorithmic assistance was solicited for the framing, design, or substance of the work, and full scientific responsibility rests with the human contributors alone.

\section{Datasets and Benchmarks}
\label{Appendix A: data}
\subsection{Mathematical Reasoning Benchmarks}

\paragraph{1. AIME24} is a dataset in evaluating the mathematical reasoning ability of models. It consists of 30 challenging math problems. All of them are from the American Invitational Mathematics Examination. The problems in the AIME24 dataset cover a wide variety of mathematical fields such as algebraic equations and geometric puzzles. Due to the difficulty characteristics and the richness of question types, it has become a popular benchmark for evaluating the reasoning performance of models, and is widely used in multiple related research experiments.

\paragraph{2. AIME25} consists of 30 challenging math problems. It is directly composed of the real questions from the American Invitational Mathematics Examination newly released in February 2025. AIME25’s knowledge areas are extremely wide. It deeply covers core mathematical sections such as algebra, geometry, number theory, and combinatorial mathematics. This characteristic enables the AIME25 dataset to effectively distinguish the mathematical reasoning abilities of different models.

\paragraph{3.  MATH500~\cite{math500}} is selected by OpenAI from the MATH evaluation dataset. It contains 500 high-difficulty math problems. These problems cover multiple mathematical fields such as algebra, geometry, calculus, and number theory. The difficulty is close to or exceeds the college level. In academic research, MATH500 dataset is often used to evaluate the performance of various reasoning models.

\paragraph{4.  GSM8K~\cite{cobbe2021gsm8k}} is an elementary school math problem dataset released by OpenAI. These problems require 2 to 8 steps to solve, mainly through a series of basic calculations to obtain the final answer. This dataset is primarily used to test the logical and mathematical abilities of models and has been applied in multiple benchmark tests.

\paragraph{5.  AMC23}  is a benchmark dataset designed to evaluate the advanced mathematical reasoning capabilities of large language models (LLMs). This dataset is primarily composed of problems from the 2023 American Mathematics Competitions (AMC 10/12).

\subsection{Knowledge-Intensive Reasoning Benchmarks}

\paragraph{HotPotQA~\cite{yang2018hotpotqa}} is a question-answering dataset for multi-hop. All the documents are sourced from Wikipedia, which provides the dataset with a rich knowledge base and relatively structured information. It is an important benchmark for evaluating the ability of LLMs to understand complex search tasks.

\paragraph{2WikiMultihopQA~\cite{xanh2020_2wikimultihop}} is a dataset specifically designed for the multi-hop questionanswering task. It aims to test and evaluate the ability of natural language processing models to answer questions that require multi-step reasoning and the integration of information from different documents.

\paragraph{Musique~\cite{trivedi2021musique}} is a question-answering dataset specifically designed for the multi-hop question-answering task. Musique aims to be a challenging benchmark for evaluating models’ multi-hop question-answering ability. It promotes the development of models from simple information retrieval to deeper semantic understanding and logical reasoning.

\paragraph{Bamboogle~\cite{Bamboogle}} is a diagnostic dataset designed to test the multi-hop reasoning capabilities of large language models, specifically focusing on questions that popular search engines (such as Google) cannot answer directly with a single search result. It consists of 125 manually curated questions that require ``bridge'' reasoning.

\subsection{Training Datasets}

\paragraph{Tool-Star-SFT-54K~\cite{toolstar}} is a supervised fine-tuning dataset designed for the "cold-start" phase of the Tool-Star framework. It contains 54,000 high-quality samples of tool-integrated reasoning, covering the complete process of autonomous tool selection, parameter invocation, and result parsing. By fine-tuning on this dataset, the model establishes foundational tool-use norms and basic multi-step reasoning capabilities, providing a stable policy base for subsequent reinforcement learning and effectively mitigating issues like reward sparsity and excessive search space during early RL stages.

\paragraph{Tool-Star-RL-10K~\cite{toolstar}} is a training dataset specifically tailored for the "Self-Critic Reinforcement Learning" stage, consisting of approximately 10,000 task samples. This dataset integrates various challenging tasks, supporting end-to-end interactive training in multi-tool environments. Through reinforcement exploration on these 10,000 samples, the model leverages the Self-Critic mechanism to optimize reasoning paths based on real tool feedback, significantly enhancing its logical rigor and self-correction capabilities in complex, long-horizon tasks.

\section{Details of Experiment Setup}
\label{Appendix:training}
\subsection{Supervised Fine-Tuning.}

In supervised fine-tuning phase, we train the model using the Llama Factory framework with a learning rate of $7 \times 10^{-6}$. We employ DeepSpeed ZeRO-3 and FlashAttention2 for optimization. The batch size is set to 128, with a weight decay of 0.1, and the model is trained for 3 epochs. We use BF16 mixed precision with a maximum input length of 4096 tokens.

\subsection{Reinforcement Learning.}

we implement the Reinforcement Learning algorithm based on the VERL framework. As Table~\ref{tab:config}, We have provided configurations for different training methods. All of our experiments were conducted on the 8*A100.

\begin{table*}[h]
    \centering
    % 如果表格过宽，可以使用 resizebox 进行缩放，或者使用 \small 缩小字号
    % \small 
    \begin{tabular}{l c c c}
        \toprule
        Config & Warm Up & Post-RL & Zero-RL \\
        \midrule
        $\gamma$ & 0.5 & 0 & 0 \\
        $\lambda$ & 0.1 & 0 & 0 \\
        n & 8 & 16 & 16 \\
        m & 8 & 0 & 0 \\
        k & 3 & 0 & 0 \\
        optimizer & AdamW & AdamW & AdamW \\
        learning rate & 1e-6 & 1e-6 & 1e-6 \\
        training data & 4000 & 10000 & 10000  \\
        total training steps & 50 & 250 & 250 \\
        training batch size & 128 & 128 & 128 \\
        PPO mini batch size & 16 & 16 & 16 \\
        max response length & 8192 & 8192 & 8192 \\
        max observation length & 512 & 512 & 512 \\
        max tool-calls & 4 & 4 & 4 \\
        retriever & local wiki & local wiki & local wiki \\
        \bottomrule
    \end{tabular}
    \caption{Hyperparameters and configurations for different Training.}
    \label{tab:config}
\end{table*}
\subsection{Self-distillation  Initialization.}

To address the model's difficulty in adhering to the specific \texttt{<python>} invocation protocol, we constructed a small-scale format alignment dataset. We sampled a subset from the SFT data, prompted the model to generate tool-integrated responses, and selected instances containing valid tool calls. We then standardized the format by replacing Markdown code blocks (\texttt{'''python ... '''}) with the target \texttt{<python> ... </python>} tags, resulting in 500 aligned examples. Since this process involved only format substitution without logical correctness verification—and the data originates from self-distillation—this stage has a negligible impact on downstream performance.

\subsection{Training Pipeine.}

\paragraph{\ourapproach{}.} In this setting, the model initially undergoes the warm-up training phase proposed in E3-TIR using the expert trajectory dataset, followed by reinforcement learning on Tool-Star10k using the Post-RL configuration.

\paragraph{SFT-then-RL.} In this setting, the model first undergoes supervised fine-tuning (SFT) on the complete 52k dataset, followed by reinforcement learning on Tool-Star10k using the same Post-RL configuration.

\paragraph{Zero-RL.} In this setting, the model is trained using the Zero-RL configuration. To ensure a fair comparison with E3-TIR and eliminate discrepancies arising from data coverage, we align the training data by including both Tool-Star10k and the expert dataset used in E3-TIR. However, we explicitly exclude the expert trajectory supervision (i.e., using only the prompts) from the expert dataset to strictly maintain the zero-shot exploration setting.

\subsection{Evaluation Setup.}

For evaluation, we extended the $\text{max\_obs\_length}$ to 1k and utilized the F1 score for search tasks, while employing ground-truth matching for mathematical reasoning tasks. The specific configurations are as follows: $\text{top\_p}=0.7$, $\text{temperature}=1.0$, $n=4$
, $\text{do\_sample}=\text{True}$.

\subsection{Reward Score.}

Reward signals serve as the optimization objective and directly guide the behavior of the policy model during training. Formally, the overall reward $R$ is defined as follows:

\begin{equation}
\begin{split}
    R &= 
    \begin{cases} 
        \max(Acc. + r_M, Acc.) & \text{If format \& Acc.} > 0 \\
        r_M & \text{If format \& Acc.} = 0 \\
        -1 & \text{Otherwise}
    \end{cases} \\
    &r_M = 
    \begin{cases} 
        0.1 & \text{If } \exists (\text{valid format}) \\
        0 & \text{Otherwise}
    \end{cases}
\end{split}
\end{equation}

For the mathematical reasoning tasks, we score them based on the degree of match between the model's returned answer and the actual answer. Specifically, we used the math\_verify library for the matching process. For the search task, we use the F1 score.

\subsection{Tool Setup.}

\paragraph{Search Engine.}  we use Wikipedia as the retrieval corpus and E5  as the retriever for four open-domain QA tasks.

\paragraph{Code Interpreter.} Code snippets generated by the language model are executed in a local environment, returning either the execution result or error messages based on correctness. Our Python interpreter implementation follows the design of ToRA , ensuring secure and accurate execution.

\section{The Algorithm Workflow Of \ourapproach}

In this section, we provide a detailed flowchart of the \ourapproach algorithm in diagram~\ref{alg:e3tir}.

\begin{algorithm*}[!h]
\small
\caption{Training Workflow of \ourapproach (Enhanced Experience Exploitation)}
\label{alg:e3tir}
\begin{algorithmic}[1]
    \normalsize
    \Require Initial Policy $\pi_{\theta}$; Reference Model $\pi_{\text{ref}}$; Expert Dataset $D_{\text{expert}}$; Self-explore Budget $N$; Expert-guided Budget $M$
    \Input Query set $Q$ from $D_{\text{expert}}$
    \State Initialize experience pools $\mathcal{D}_{\text{self}} \gets \varnothing, \mathcal{D}_{\text{exp}} \gets \varnothing$
    
    \For{each training iteration $i = 1 \dots K$}
        \algcomment{Phase 1: Expert Experience-Based Rollout}
        \For{each query $q \in Q$}
            \State \textbf{Step 1.1: Anchor Selection}
            \State Calculate step-wise entropy $h_t$ on expert trajectory $\mathcal{H}^*$
            \State Select top-$k$ high-entropy prefixes as Anchors $\mathcal{T}_{\text{anchor}}$
            
            \State \textbf{Step 1.2: Hybrid Sampling}
            \State Sample $N$ trajectories from scratch: $\mathcal{D}_{\text{self}} \gets \mathcal{D}_{\text{self}} \cup \text{Rollout}(\pi_{\theta}, q, N)$
            \For{each anchor $t \in \mathcal{T}_{\text{anchor}}$}
                \State Calculate exploration probability $P_t \propto (h_t - h_{t-1})$
                \If{$P_t > \gamma$ and $|\mathcal{D}_{\text{exp}}| < M$}
                    \State Branch sample from prefix $\mathcal{H}^*_{<t}$: $\mathcal{D}_{\text{exp}} \gets \mathcal{D}_{\text{exp}} \cup \text{Rollout}(\pi_{\theta}, \mathcal{H}^*_{<t})$
                \EndIf
            \EndFor
        \EndFor

        \algcomment{Phase 2: Dynamic Experience Filtering}
        \State Discard expert trajectories with zero reward variance: $\mathbb{V}[r] = 0$
        \State When the variance is 0, only one trajectory is retained.
        \If{$\max(r(\mathcal{D}_{\text{self}})) \ge \max(r(\mathcal{D}_{\text{exp}}))$}
            \State $\mathcal{D}_{\text{train}} \gets \mathcal{D}_{\text{self}}$ \Comment{Self-exploration is sufficient}
        \Else
            \State $\mathcal{D}_{\text{train}} \gets \mathcal{D}_{\text{self}} \cup \mathcal{D}_{\text{exp}}$ \Comment{Need expert guidance}
        \EndIf

        \algcomment{Phase 3: Mix Policy Optimization}
        \For{each trajectory $\tau \in \mathcal{D}_{\text{train}}$}
            \State Calculate Global Advantage: $\hat{A}_{\text{global}} \gets (r(\tau) - \mu_{\text{global}}) / \sigma_{\text{global}}$
            \If{$\tau \in \mathcal{D}_{\text{exp}}$}
                \State Calculate Tree Advantage: $\hat{A}_{\text{tree}} \gets (r(\tau) - \mu_{\text{tree}}) / \sigma_{\text{tree}}$
                \State Final Advantage $\hat{A} \gets \hat{A}_{\text{global}} + \hat{A}_{\text{tree}}$
            \Else
                \State Final Advantage $\hat{A} \gets \hat{A}_{\text{global}}$
            \EndIf
            
            \State \textbf{Advantage-Aware Gradient Detach:}
            \If{$\tau$ shares prefix with expert AND $\hat{A} < 0$}
                 \State Detach gradients for the shared prefix steps (Block optimization on prefix)
            \EndIf
        \EndFor
        
        \State Update $\pi_{\theta}$ by maximizing $\mathcal{J}_{\text{Hybrid}}$ with off-policy correction
    \EndFor
    \State \Return Optimized Policy $\pi_{\theta}$
\end{algorithmic}
\end{algorithm*}

\section{Theoretical Analysis}
\label{sec:appendix_theory}

In this section, we provide a theoretical justification for the effectiveness of the $E^3$-TIR framework. Unlike token-level generation tasks, Tool-Integrated Reasoning (TIR) is modeled as a discrete decision-making process over a sequence of \textbf{nodes} (interaction turns). We analyze how Expert-Guided Branching mitigates the exploration complexity in this multi-turn setting.

\subsection{Problem Formulation}

Following the definitions in Section 2.1, we consider the reasoning process as a trajectory of nodes.

\begin{definition}[Node and Trajectory]
Let a reasoning task utilize a maximum of $T$ interaction turns. A trajectory is a sequence of nodes $\mathcal{H} = \{u_1, u_2, ..., u_T\}$, where each node $u_t = (\tau_t, a_t, o_t)$ represents the triplet of thought, action, and observation at turn $t$. The policy $\pi_{\theta}$ generates the next node conditioned on the history of previous nodes: $u_t \sim \pi_{\theta}(\cdot | \mathcal{H}_{<t})$.
\end{definition}

\begin{assumption}[Step-wise Difficulty]
For a complex reasoning task, obtaining a successful outcome requires a sequence of correct decisions. We assume the ``student'' policy $\pi_{\theta}$ (before alignment) has a bounded probability $p \ll 1$ of independently generating the optimal next node $u^*_t$ given the correct history $\mathcal{H}^*_{<t}$:
\begin{equation}
    P(u_t = u^*_t | \mathcal{H}^*_{<t}, \pi_{\theta}) \le p
\end{equation}
Here, $p$ represents the \textbf{turn-level} success rate, not token-level probability.
\end{assumption}

\subsection{Failure Analysis of Zero-RL}
\label{subsec:theory_zero_rl}

In the Zero-RL setting (or pure Self-Exploration), the agent must successfully navigate from the initial state $u_0$ to a terminal success state within $T$ turns.

\begin{theorem}[Curse of Horizon in Zero-RL]
For a task requiring $T$ critical reasoning turns, the probability of a randomly initialized policy $\pi_{\theta}$ generating a fully correct trajectory $\mathcal{H}_{succ}$ is exponentially small with respect to the number of turns $T$.
\end{theorem}

\begin{proof}
The probability of generating a complete successful trajectory is the product of the success probabilities at each node. Assuming independence for the lower bound analysis:
\begin{equation}
    P(\text{Success}|\text{Zero-RL}) = \prod_{t=1}^{T} P(u^*_t | \mathcal{H}^*_{<t}, \pi_{\theta}) \le p^T
\end{equation}
Since $p < 1$, as the required turns $T$ increase, $p^T$ decays exponentially. For instance, if an agent has a 50\% chance of making a correct tool call at each step ($p=0.5$) and the task requires 10 steps ($T=10$), the success rate is $0.5^{10} \approx 0.0009$. This explains the ``Inefficient Exploration'' (Figure 1) where the agent fails to collect positive feedback during the cold-start phase.
\end{proof}

\subsection{Success Guarantee of $E^3$-TIR}
\label{subsec:theory_e3tir}

$E^3$-TIR addresses this challenge by utilizing \textbf{Expert-Guided Branch Sampling} (Section 4.1.1). We select an expert prefix containing $k$ nodes, denoted as $\mathcal{H}^*_{<k}$ (where $k < T$), and perform $G$ independent branch rollouts from this anchor node.

\begin{theorem}[Efficiency of Node-level Branching]
Let the agent branch from an expert anchor node at turn $k$. The effective horizon is reduced to $T_{eff} = T - k$ turns. With $G$ parallel branches, the probability of discovering a successful path is significantly amplified.
\end{theorem}

\begin{proof}
Given the expert history of $k$ nodes, the agent only needs to complete the remaining $T-k$ turns correctly.

1. \textbf{Single Branch Probability:} The probability of a single rollout successfully completing the remaining turns is:
\begin{equation}
    P_{single} = p^{T-k}
\end{equation}
Crucially, since we branch at deep nodes (e.g., $k \approx T/2$ or $k \approx 3T/4$), the remaining horizon $T-k$ is small, implying $p^{T-k} \gg p^T$.

2. \textbf{Group Success Probability:} With $G$ independent branches sampled from the policy $\pi_{\theta}(\cdot | \mathcal{H}^*_{<k})$, the probability that \textit{at least one} branch reaches the solution is:
\begin{equation}
    P(\text{Success}|E^3\text{-TIR}) = 1 - (1 - p^{T-k})^G
\end{equation}
\end{proof}

\paragraph{Mitigating Low-Entropy Collapse.}
Traditional SFT+RL often collapses into repetitive, simple patterns (Fast Mode) to minimize turn length. By forcing the agent to branch from expert nodes $\mathcal{H}^*_{<k}$—which are by definition diverse and high-quality—we enforce exploration in the valid reasoning subspace. The condition for $E^3$-TIR to outperform Zero-RL is simply $(1 - p^{T-k})^G > p^T$, which holds true for almost any $k > 0$ and $G \ge 1$.

\subsection{Numerical Illustration}
Consider a complex mathematical reasoning task requiring $T=8$ tool-use turns. Let the base model's turn-level correctness probability be $p=0.4$.
\begin{itemize}
    \item \textbf{Zero-RL:} Success probability $\approx 0.4^8 \approx 0.00065$ (approx. 1 in 1500). The reward signal is extremely sparse.
    \item \textbf{$E^3$-TIR:} We use an expert prefix of $k=6$ turns (leaving 2 turns to solve) and sample $G=8$ branches.
    \begin{itemize}
        \item Success rate per branch: $0.4^2 = 0.16$.
        \item Total success probability: $1 - (1 - 0.16)^8 \approx 1 - 0.25 \approx 75\%$.
    \end{itemize}
\end{itemize}
This dramatic increase from $<0.1\%$ to $75\%$ demonstrates why anchoring at deep nodes allows $E^3$-TIR to achieve ``Fast Convergence'' and break the exploration bottleneck.

\section{Supplement Experimental Results}
\label{Appendix:supple}
\subsection{Statistic Analyse of Problem in Current Training Paradigms.}

\paragraph{Zero-RL Exploration Efficiency and Mode Analysis.} We analyzed the behavioral distribution of Zero-RL during Tool-Integrated Reasoning (TIR) tasks:

\begin{itemize}
    \item \textbf{Tool Call Failure Statistics (Fig.~\ref{fig:motivation} (left,a)):} We tracked the ratio of valid tool calls and the distribution of consecutive code errors during the initial training phase (first 40 steps). A trajectory is considered valid only if it contains code blocks that are both syntactically correct and executable by the interpreter.
    
    \item \textbf{Reasoning Mode Degradation (Fig.~\ref{fig:motivation} (left,b)):} We monitored the prevalence of the direct action mode (i.e., invoking tools without intermediate reasoning). We utilize the React-Mode Ratio (RMR) as a metric, defined as:
    \begin{equation}
        \text{RMR} = 10000*\frac{N_{\text{tool}}}{L_{\text{text}}} * L_{\text{avg}}
    \end{equation}
    where $N_{\text{tool}}$ denotes the total count of tool invocations,  $L_{\text{text}}$ is the average response length and $L_{\text{text}}$ represents the token length of the response excluding tool generation segments. A higher RMR value implies that the agent tends to execute actions directly with minimal intermediate reasoning (Chain-of-Thought), indicating a degradation in reasoning capability.
\end{itemize}

\paragraph{SFT+RL Policy Diversity and Entropy Analysis.} To investigate the phenomenon where SFT+RL starts strong but subsequently stagnates or collapses, we conducted the following analyses:

\begin{itemize}
    \item \textbf{Low Entropy Collapse (Fig.~\ref{fig:motivation} (right,a)):} We tracked the policy entropy $h_t$ at each RL step. Results show a sharp decline in entropy as training progresses, indicating that the model collapses into repetitive, simplistic modes and loses the capacity to explore complex reasoning paths.
    
    \item \textbf{Capability Rigidity and Distribution Shift (Fig.~\ref{fig:motivation} (right,b)):} We visualized the representation of rollout sampling by embedding generated thought paths using Qwen-Embedding, followed by PCA dimensionality reduction and DBSCAN clustering. Comparing the trajectory diversity of Zero-RL and SFT models at the same training step, the results clearly show that the number of sample clusters for the SFT model decreases significantly in later stages.  
\end{itemize}

\subsection{Comparison of Expert Prefixes and Self-Exploration Sampling.}

To investigate the variations in reward standard deviation and entropy under different sampling paradigms, we conducted two distinct RL training sessions using Qwen2.5-3B-Instruct, with hyper parameters fixed at $N=8$ and $M=8$. The experiments differed in the source of the $M$ branches: one utilized expert prefixes, while the other relied on self-generated trajectories. As this study serves as a preliminary analysis, we limited training to 50 steps and recorded: 

(1) The standard deviation of rewards within trajectory trees sharing a common prefix.

(2) The response entropy arising from branching at different turns, calculated specifically over the first 20 generated tokens.

\subsection{Warm-up Stage Cost and Efficiency Analysis}
\label{app:cost}

To quantitatively evaluate the trade-off between pre-training costs and the subsequent reinforcement learning (RL) efficiency, we introduce a comprehensive metric, the \textbf{Efficient Cost-Benefit Ratio ($\text{ROI}$)}. This metric assesses how the \ourapproach paradigm, utilizing varying scales of expert data and warm-up durations, accelerates the convergence of the Post-RL stage while maintaining or enhancing final performance.

We formulate the training process into two distinct stages: the Warm-up stage (utilizing \ourapproach) and the Post-RL stage. The relevant variables are defined as follows:

\begin{itemize}
    \item $X$: The training duration (or steps) allocated to the Warm-up stage.
    \item $Y$: The size of expert trajectories utilized during the Warm-up stage.
    \item $Z$: The training duration (or steps) required to reach convergence in the Post-RL stage.
    \item $Z_{\text{base}}$: The baseline time required for the model to reach convergence without any warm-up (i.e., Zero-RL setting).
    \item $V$: The performance improvement coefficient, defined as the ratio of the final performance of the proposed method to the Zero-RL baseline (e.g., $V=1.12$ implies a 12\% improvement).
\end{itemize}

The $\text{ROI}$ is designed to reward methods that achieve higher performance gains ($V$) and faster total convergence ($Z + \alpha X$) while minimizing reliance on extensive expert data ($Y$). The formula is defined as:

\begin{equation}
    \text{ROI} =\underbrace{V}_{\text{Performance}} \times \underbrace{\frac{Z_{\text{base}}}{Z + \alpha \cdot X}}_{\text{Time Efficiency}} \times \underbrace{1-\frac{Y}{Y_{\text{ref}}}}_{\text{Data Efficiency}}
\end{equation}

where:
\begin{itemize}
    \item $\alpha$ represents the relative computational cost coefficient between the Warm-up stage and the Post-RL stage. Since the Warm-up stage (typically use Expert Prefixes) is computationally less intensive per step than RL (e.g., GRPO). Based on the update statistics of each step during the training process, we set $\alpha =0.62$ to normalize the warm-up cost into equivalent RL units.
    \item $Y_{\text{ref}}$ is the reference data scale constant (set to $52k$ in our experiments). This term normalizes the data efficiency, penalizing the excessive use of expert trajectories.
    \item  \text{Norm} represents the operation of mapping this set of data to the range of 0 to 1.
\end{itemize}

\paragraph{Interpretation.} A higher $\text{ROI}$ indicates that the E3-TIR paradigm successfully leverages a small amount of expert data and a short warm-up phase to significantly reduce the expensive Post-RL training time ($Z \ll Z_{\text{base}}$), thereby achieving a superior overall cost-benefit trade-off.

\begin{table}[!h]
    \centering
    \fontsize{11pt}{14pt}\selectfont % 在外部统一设置字号
    % 第一个表格
    \begin{minipage}{0.48\textwidth}
        \centering
        % 移除 \resizebox，直接写 tabular
        \begin{tabular}{l|cccc}
        \toprule
        \textbf{Data Size} & \multicolumn{4}{c}{\textbf{Expert Trajectories}} \\
        \cmidrule(lr){2-5}
        \textbf{Warm up} & 2k & 4k & 8k & 10k \\
        \midrule
        10 steps & 43.0 & 43.4 & 42.7 & 42.9 \\
        30 steps & 44.3 & 44.6 & 43.8 & 44.5 \\
        50 steps & 43.7 & 46.7 & 46.3 & 46.0 \\
        70 steps & 44.3 & 45.6 & 45.5 & 46.1 \\
        \bottomrule
        \end{tabular}
        \caption{Model performance under different expert trajectories size and Warm up steps on Qwen2.5-3B-Instruct.}
        \label{tab:data1}
    \end{minipage}
    \hfill 
    % 第二个表格
    \begin{minipage}{0.48\textwidth}
        \centering
        \begin{tabular}{l|cccc}
        \toprule
        \textbf{Data Size} & \multicolumn{4}{c}{\textbf{Expert Trajectories}} \\
        \cmidrule(lr){2-5}
        \textbf{Warm up} & 2k & 4k & 8k & 10k \\
        \midrule
        10 steps & 200 & 210 & 190 & 200 \\
        30 steps & 180 & 170 & 160 & 180 \\
        50 steps & 140 & 120 & 120 & 130 \\
        70 steps & 110 & 100 & 100 & 110 \\
        \bottomrule
        \end{tabular}
        \caption{Post-RL convergence steps under different expert trajectories size and Warm up steps on Qwen2.5-3B-Instruct.}
        \label{tab:data2} % 注意修改 label 避免重复
    \end{minipage}
\end{table}

\subsection{Comparison with Offline Fusion Baselines}
\label{app:mix}

We benchmark the mixed strategy of E3-TIR against other policy fusion paradigms by comparing it with two state-of-the-art baselines:

\begin{itemize}
    \item \textbf{Luffy~\cite{luffy}:} This method formulates a dynamically adjustable objective that balances ``teacher imitation'' and ``self-exploration.'' Utilizing Group Relative Policy Optimization (GRPO), Luffy directly incorporates high-quality off-policy reasoning trajectories generated by a strong teacher (DeepSeek-R1) into the training batch, computing advantages in conjunction with the model's own on-policy trajectories.

    \item \textbf{HPT~\cite{hpt}:} This approach implements a real-time gating mechanism based on the rollout accuracy $P$ for each specific query. It dynamically toggles between learning signals using a threshold $\gamma$: applying pure SFT with demonstrations when performance is low ($P \le \gamma$), and switching to pure GRPO for exploration when performance is high ($P > \gamma$). HPT achieves a minimalist fusion—characterized as ``one formula, two signals, automatic switching''—without requiring manual weight tuning or auxiliary models.
\end{itemize}

\subsection{Expert Trajectory Analysis during the Warm-up stage}

We analyze the evolution of the mixed strategy during the warm-up phase, focusing on sampling ratios and relative rewards. Figure~\ref{fig:app1} (left) illustrates a decline in expert-guided samples concurrent with a rise in fully on-policy generation as training progresses. This trend reflects a reduced reliance on expert trajectories, validating \ourapproach's two-stage training design. Additionally, Figure~\ref{fig:app1} (right) demonstrates that the mixed strategy consistently yields higher rewards than the on-policy approach, indicating that expert prefixes effectively enhance the probability of generating optimal trajectories.

\begin{figure*}[!h]
	\centering
	\includegraphics[width=1.4\columnwidth]{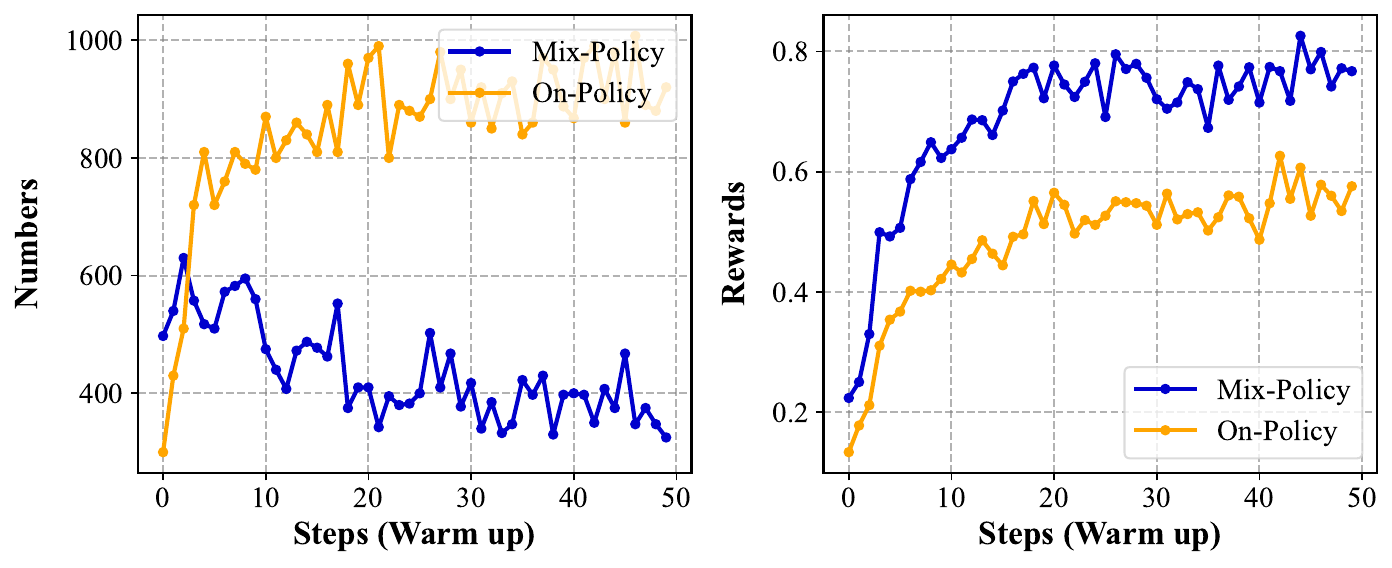}
    	\caption{Variation of the sample size and rewards in the preheating stage of the mixed strategy.}
        \label{fig:app1}
\end{figure*}

\section{Instruction Template}
\label{Appendix:template}

\subsection{Instructions for Mathematical Reasoning Evaluation}

\begin{tcolorbox}[
    colframe=gray,       % Border color
    colback=gray!5!white,             % Background color
    coltitle=white,                   % Title font color,
    coltext=black,                    % Text color,
    fonttitle=\bfseries,              % Bold title font,
    title=Instruction for Sampling and Reasoning Computational Datasets,  % Title content,
    boxrule=1pt,                      % Border width,
    arc=2mm,                          % Rounded corners,
    width=\linewidth,                 % Width,
    left=7pt,                         % Left margin,
    right=7pt,                        % Right margin,
    top=5pt,                          % Top margin,
    bottom=5pt                        % Bottom margin
]
\fontsize{10pt}{10pt}\selectfont
    You are a helpful assistant that can solve the given question step by step with the help of the python interpreter tool. 
    Given a question, you need to first think about the reasoning process in the mind and then provide the answer. 
    During thinking, you can invoke the python interpreter tool to calculate the math problem for fact information about specific topics if needed. 
  
    The reasoning process and answer are enclosed within <think> </think> and <answer> </answer> tags respectively. 
    For example, <think> This is the reasoning process. </think> <python> python code here </python> <result> python interpreter result here </result> 
    <think> This is the reasoning process. </think> <answer> The final answer is  \verb|\|boxed\{answer\ here\} </answer>. 
    In the last part of the answer, the final exact answer is enclosed within boxed{} with latex format. \\\{question\}
\end{tcolorbox}

\subsection{Instructions for Knowledge-Intensive Reasoning Evaluation}

\begin{tcolorbox}[
    colframe=gray,       % Border color
    colback=gray!5!white,             % Background color
    coltitle=white,                   % Title font color,
    coltext=black,                    % Text color,
    fonttitle=\bfseries,              % Bold title font,
    title=Instruction for Sampling and Reasoning Computational Datasets,  % Title content,
    boxrule=1pt,                      % Border width,
    arc=2mm,                          % Rounded corners,
    width=\linewidth,                 % Width,
    left=7pt,                         % Left margin,
    right=7pt,                        % Right margin,
    top=5pt,                          % Top margin,
    bottom=5pt                        % Bottom margin
]
\fontsize{10pt}{10pt}\selectfont
    You are a helpful assistant that can solve the given question step by step with the help of the wikipedia search tool. 
    Given a question, you need to first think about the reasoning process in the mind and then provide the answer. 
    During thinking, you can invoke the wikipedia search tool to search for fact information about specific topics if needed. 

    The reasoning process and answer are enclosed within <think> </think> and <answer> </answer> tags respectively, 
    and the search query and result are enclosed within <search> </search> and <result> </result> tags respectively. 
    For example, <think> This is the reasoning process. </think> <search> search query here </search> <result> search result here </result> 
    <think> This is the reasoning process. </think> <answer> The final answer is \verb|\|boxed\{answer\ here\} </answer>.
     
    In the last part of the answer, the final exact answer is enclosed within boxed{} with latex format. \\\{question\}
\end{tcolorbox}

\subsection{Instructions for Reinforcement Learning Training}

\begin{tcolorbox}[
    colframe=gray,       % Border color
    colback=gray!5!white,             % Background color
    coltitle=white,                   % Title font color,
    coltext=black,                    % Text color,
    fonttitle=\bfseries,              % Bold title font,
    title=Instruction for Sampling and Reasoning Computational Datasets,  % Title content,
    boxrule=1pt,                      % Border width,
    arc=2mm,                          % Rounded corners,
    width=\linewidth,                 % Width,
    left=7pt,                         % Left margin,
    right=7pt,                        % Right margin,
    top=5pt,                          % Top margin,
    bottom=5pt                        % Bottom margin
]
\fontsize{10pt}{10pt}\selectfont
    You are a helpful assistant that can solve the given question step by step with the help of the wikipedia search tool and python interpreter tool. 
    
    Given a question, you need to first think about the reasoning process in the mind and then provide the answer. 
    During thinking, you can invoke the wikipedia search tool to search and python interpreter tool to calculate the math problem for fact information about specific topics if needed. 
    The reasoning process and answer are enclosed within <think> </think> and <answer> </answer> tags respectively, and the search query and result are enclosed within <search> </search> and <result> </result> tags respectively. 
    
    For example, <think> This is the reasoning process. </think> <search> search query here </search> <result> search result here </result> <think> This is the reasoning process. </think> 
    <python> python code here </python> <result> python interpreter result here </result> <think> This is the reasoning process. </think> 
    <answer> The final answer is  \verb|\|boxed\{answer\ here\} </answer>. In the last part of the answer, the final exact answer is enclosed within  boxed{} with latex format.
\end{tcolorbox}

\section{Case Study}
\label{Appendix:case}

\begin{table*}[!t]
    \centering
    \caption{An example from \ourapproach on the Musique dataset, with special symbols used in think content, search queries, returned results and final answer highlighted with \purple{purple box}, \green{green box} and \red{red box}, respectively.}
    \fontsize{10pt}{11pt}\selectfont
    \begin{tabular}{p{0.98\linewidth}}
    \midrule
        \rowcolor{gray!20}\textbf{Example from Musique} \\
    \midrule
        \textbf{Question:} \\
        What country has the border troops of the country of the literature of the country of Karl Maron citizenship?\\
        \textbf{Labeled Answer:} \\
        East Germany \\
    \midrule
        \rowcolor{gray!20}
        \textbf{Model Output by \ourapproach (Qwen2.5-3B)}\\
    \midrule
        \purple{<think>} To solve this question, I need to find out the country whose border troops have the citizenship of the country where the literature of the country of Karl Maron is from.
        \\
        Step 1: Identify the country where Karl Maron is from.\\
        Step 2: Identify the country whose literature is from the country of Step 1.\\
        Step 3: Identify the citizenship of the country from Step 2.\\
        Step 4: Identify the country whose border troops have the citizenship from Step 3.\\

        Step 1: Identify the country where Karl Maron is from.\\

        I'll use a query to search for this information:\purple{</think>}\\\green{<search>} Karl Maron country of citizenship \green{</search>}\\\red{<result>}
        Doc 1 [id: 11966372] Title: Karl Maron
        Karl Maron (27 April 1903 – 2 February 1975) was a German politician, who served as the interior minister of East Germany. He also assumed different posts in East Germany's government...
\\
        Doc 2 [id: 9746062] Title: Karl Maron
        ...Maron was the step-father of author Monika Maron. Karl Maron married her mother in 1955. He died in 1975. Karl Maron Karl Maron (27 April 1903 – 2 February 1975) was a German politician, who served as the interior minister of East Germany...
        \red{</result>} \\\purple{<think>} Step 2: Identify the country whose literature is from the country of Step 1.
\\
        I'll use a query to search for this information:\purple{</think>}\\\green{<search>} country where the literature of East Germany is from \green{</search>}\\\red{<result>}
        Doc 1 [id: 5484532] Title: Literature of East Germany
        Literature of East Germany East German literature is the literature produced in East Germany from the time of the Soviet occupation in 1945 until the end of the communist government in 1990...
\\
        Doc 2 [id: 12666639] Title: Literature of East Germany
        the literature produced in this period among the best of the entire GDR. The 1990s saw the reunification of East and West Germany and the abrupt demise of the dream of a German "socialist utopia."...
        \red{</result>} \\\purple{<think>} Step 3: Identify the citizenship of the country from Step 2.
\\
        I'll use a query to search for this information:\purple{</think>}\\\green{<search>} literature of East Germany country \green{</search>}\\\red{<result>}
        Doc 1 [id: 5484532] Title: Literature of East Germany
        Literature of East Germany East German literature is the literature produced in East Germany...
\\
        Doc 3 [id: 16627150] Title: Culture of East Germany
        Culture of East Germany The culture of East Germany varied throughout the years due to the political and historical events that took place in the 20th century...
        \red{</result>} \\\purple{<think>} This is last step and I need to give the final answer, which is the country whose border troops have the citizenship of the country where the literature of the country of Karl Maron is from.\purple{</think>}\\\red{<answer>}\\The final answer is \boxed{East Germany}\\\red{</answer>}\\
    \bottomrule
    \end{tabular}
\end{table*}

\begin{table*}[!t]
    \centering
    \caption{An example from \ourapproach on the AIME25 dataset, with special symbols used in think content.}
    \fontsize{9pt}{11pt}\selectfont
    \begin{tabular}{p{0.98\linewidth}}
    \midrule
        \rowcolor{gray!20}\textbf{Example from Dataset} \\
    \midrule
        \textbf{Question:} \\
        Find the number of ordered pairs $(x,y)$, where both $x$ and $y$ are integers between $-100$ and $100$ inclusive, such that $12x^2-xy-6y^2=0$.\\
        \textbf{Labeled Answer:} 117\\
    \midrule
        \rowcolor{gray!20}
        \textbf{Model Output by  \ourapproach (Qwen2.5-3B)}\\
    \midrule

        \purple{<think>} To solve the problem, we need to find the number of ordered pairs \((x, y)\) such that both \(x\) and \(y\) are integers between \(-100\) and \(100\) inclusive, and they satisfy the equation \(12x^2 - xy - 6y^2 = 0\).\\
        First, let's rearrange the equation to make it easier to solve:
        \[12x^2 - xy - 6y^2 = 0.\]
        This is a quadratic equation in terms of \(x\). We can use the quadratic formula to solve for \(x\):
        \[12x^2 - xy - 6y^2 = 0.\]
        The quadratic formula for \(ax^2 + bx + c = 0\) is given by:
        \[x = \frac{-b \pm \sqrt{b^2 - 4ac}}{2a}.\]
        In our equation, \(a = 12\), \(b = -y\), and \(c = -6y^2\). So the solutions for \(x\) are:
        \[x = \frac{y \pm \sqrt{y^2 + 24 \cdot 6y^2}}{2 \cdot 12} = \frac{y \pm \sqrt{y^2 + 144y^2}}{24} = \frac{y \pm \sqrt{145y^2}}{24} = \frac{y \pm y\sqrt{145}}{24}.\]
        This gives us two solutions:
        \[x = \frac{y(1 + \sqrt{145})}{24} \quad \text{and} \quad x = \frac{y(1 - \sqrt{145})}{24}.\]
        For \(x\) to be an integer, \(\frac{y(1 + \sqrt{145})}{24}\) and \(\frac{y(1 - \sqrt{145})}{24}\) must be integers. Since \(\sqrt{145}\) is not an integer, \(y\) must be a multiple of 24 for \(\frac{y(1 + \sqrt{145})}{24}\) and \(\frac{y(1 - \sqrt{145})}{24}\) to be integers. Let's denote \(y = 24k\), where \(k\) is an integer. Then the solutions for \(x\) become:
        \[x = \frac{24k(1 + \sqrt{145})}{24} = 1 + \sqrt{145}k \quad \text{and} \quad x = \frac{24k(1 - \sqrt{145})}{24} = 1 - \sqrt{145}k.\]
        Since \(\sqrt{145}\) is not an integer, \(k\) must be such that \(1 + \sqrt{145}k\) and \(1 - \sqrt{145}k\) are integers. This means \(k\) must be an integer that makes \(1 + \sqrt{145}k\) and \(1 - \sqrt{145}k\) integers. The only way this can happen is if \(k = 0\), but \(k = 0\) would make \(x = 1\). However, we need to check if there are any other values of \(k\) that might work.\\
        Let's check the possible values of \(k\) by iterating over the range of \(y\) from \(-100\) to \(100\). We will use Python to check which pairs \((x, y)\) satisfy the equation.\\
        Let's write the Python code to find the number of such pairs \((x, y)\).
        \purple{</think>}\\\blue{<python>}

        \begin{minipage}{\linewidth} % 使用 minipage 防止表格断裂
        \begin{lstlisting}[style=python]
# Define the range for x and y
range_min = -100
range_max = 100

# Initialize a counter for the number of valid pairs
count = 0

# Iterate over all possible values of x and y
for x in range(range_min, range_max + 1):
    for y in range(range_min, range_max + 1):
        if 12*x**2 - x*y - 6*y**2 == 0:
            count += 1

# Output the result
print(count)
        \end{lstlisting}
        \end{minipage}
        
        \blue{</python>}\\\red{<result>}117\red{</result>} \\\purple{<think>}
        
        The number of ordered pairs \((x, y)\), where both \(x\) and \(y\) are integers between \(-100\) and \(100\) inclusive, such that \(12x^2 - xy - 6y^2 = 0\) is \(\boxed{117}\).
        
        \purple{</think>}\\\red{<answer>}The final answer is \boxed{117}\red{</answer>}\\
    \bottomrule
    \end{tabular}
\end{table*}

\begin{table*}[!t]
    \centering
    \caption{An example from \ourapproach on the AIME24 dataset, with special symbols used in think content.}
    \fontsize{9pt}{11pt}\selectfont
    \begin{tabular}{p{0.98\linewidth}}
    \midrule
        \rowcolor{gray!20}\textbf{Example from Dataset} \\
    \midrule
        \textbf{Question:} 
        Let $p$ be the least prime number for which there exists a positive integer $n$ such that $n^{4}+1$ is divisible by $p^{2}$. Find the least positive integer $m$ such that $m^{4}+1$ is divisible by $p^{2}$.\textbf{Labeled Answer:} 110\\
    \midrule
        \rowcolor{gray!20}
        \textbf{Model Output by \ourapproach (Qwen2.5-3B)}\\
    \midrule

        \purple{<think>} To solve this problem, we need to follow these steps:

1. Identify the smallest prime number $p$ for which there exists a positive integer $n$ such that $n^4 + 1$ is divisible by $p^2$.
2. Once we find this prime $p$, we need to find the smallest positive integer $m$ such that $m^4 + 1$ is divisible by $p^2$.

Let's start by iterating over prime numbers and checking the condition for $n^4 + 1$ being divisible by $p^2$.

We will use Python to find the smallest prime $p$ and then the smallest $m$.First, we will write a function to check if a number is prime and then iterate over prime numbers to find the smallest $p$.  Let's write the code to achieve this.
        \purple{</think>} \blue{<python>}

        \begin{minipage}{\linewidth}
        \begin{lstlisting}[style=python]
from sympy import isprime, primerange
# Function to find the smallest prime p
def find_smallest_prime_p():
    p = 2
    while True:
        found = False
        # Check for the first 1000 values of n to find if there exists an n 
        # such that (n^4 + 1) is divisible by p^2
        for n in range(1, 1000):
            if (n**4 + 1) % (p**2) == 0:
                found = True
                break
        if found:
            return p
        p = next_prime(p)

# Find the smallest prime p
smallest_prime_p = find_smallest_prime_p()
print((smallest_prime_p, smallest_m))
        \end{lstlisting}
        \end{minipage}
        \blue{</python>}\red{<result>}Code execution failed: Traceback (most recent call last):
  File ``<string>'', line 26, in <module>
  File ``<string>'', line 15, in find\_smallest\_prime\_p
NameError: name `next\_prime' is not defined\red{</result>} \purple{<think>}
        
        It seems that the previous approach had some issues. Let's solve this problem step-by-step more methodically.

1. **Identify the smallest prime $p$ such that there exists a positive integer $n$ for which $n^4 + 1$ is divisible by $p^2$.**

2. **Once we find $p$, determine the smallest integer $m$ such that $m^4 + 1$ is divisible by $p^2$.**

We will start by iterating over prime numbers and check for each prime $p$ whether there exists an $n$ such that $n^4 + 1 \equiv 0 \pmod{p^2}$. This means $n^4 \equiv -1 \pmod{p^2}$. We will then check for the smallest $m$ that satisfies this condition for $p^2$.

Let's write the corrected and more analytical code to solve this:

1. Generate a list of prime numbers.
2. For each prime $p$, check if there exists an $n$ such that $n^4 \equiv -1 \pmod{p^2}$.
3. Once we find the smallest such $p$, find the smallest $m$ such that $m^4 \equiv -1 \pmod{p^2}$.
        \purple{</think>} \blue{<python>}

        \begin{minipage}{\linewidth}
        \begin{lstlisting}[style=python]
from sympy import isprime, primerange, next_prime
# Function to find the smallest m for a given prime p
def find_smallest_m_for_prime_p(p):
    m = 1
    while True:
        if (pow(m, 4, p**2) + 1) % (p**2) == 0:
            return m
        m += 1
# Iterate over primes
primes = list(primerange(2, 10000))
for p in primes:
    found = False
    # Check for an n such that n^4 + 1 is divisible by p^2
    # The search range for n can be up to p^2
    for n in range(1, 1000): 
        if (pow(n, 4, p**2) + 1) % (p**2) == 0:
            found = True
            break
    if found:
        smallest_prime_p = p
        break
smallest_m = find_smallest_m_for_prime_p(smallest_prime_p)
print((smallest_prime_p, smallest_m))
        \end{lstlisting}
        \end{minipage}
        \blue{</python>}\red{<result>}(17, 110)\red{</result>}\purple{<think>}
        
        The smallest prime number $p$ for which there exists a positive integer $n$ such that $n^4 + 1$ is divisible by $p^2$ is $17$.
The least positive integer $m$ such that $m^4 + 1$ is divisible by $17^2 = 289$ is $110$.Thus, the final answer is $\boxed{110}$.
        \purple{</think>}\\\red{<answer>}The final answer is \boxed{110}\red{</answer>}\\
    \bottomrule
    \end{tabular}
\end{table*}

\end{document}

%% file: tables/compare_to_algothrim.tex
\begin{table*}[t]
\centering
\caption{Overall results on 10 challenging reasoning tasks. The top two results are highlighted in \textbf{bold} and \underline{underlined}. The abbreviations of the datasets: Math (Math500), HQA (HotpotQA), 2Wiki. (2wikiMultiHopQA), MuSi. (MuSiQue), SimQA (SimpleQA) and Bamb (Bamboogle).}
\label{tab:compare_modified}
\setlength\tabcolsep{2.0pt}
\renewcommand{\arraystretch}{1}
\fontsize{8.5pt}{11pt}\selectfont
\begin{tabular}{p{2.8cm}cccccccccccc}
\toprule
\multirow{2}[2]{*}{\textbf{Method}} & \multicolumn{5}{c}{\textbf{Mathematical Reasoning}} & \multicolumn{5}{c}{\textbf{Knowledge-Intensive Reasoning}} & \multirow{2}[2]{*}{\textbf{Avg.}} & \multirow{2}[2]{*}{\textbf{RMR.}} \\
\cmidrule(lr){2-6} \cmidrule(lr){7-11}
& AIME24 & AIME25 & AMC23 & GSM8K & MATH & 2Wiki. & Bamb. & HQA & MuSi. & SimQA & & \\
\midrule
\textbf{Qwen2.5-3B-Instruct} & 4.5 & 8.6 & 38.3 & 65.3 & 81.0 & 19.2 & 27.5 & 11.5 & 6.3 & 5.9 & 26.8 & $-$ \\
\quad + Only SFT & 6.0 & 10.8 & 35.6 & 58.3 & 81.0 & 30.3 & 35.2 & 35.8 & 16.3 & 19.5 & 32.9 & 7.4 \\
\quad + SFT-then-RL & 11.2 & \underline{16.5} & \underline{51.8} & 69.1 & 85.1 & \textbf{54.2} & 47.3 & \underline{50.5} & \underline{25.3} & 30.6 & \underline{44.2} & \textbf{6.1} \\
\quad + Zero-RL & \underline{12.0} & 13.3 & 48.7 & \underline{71.9} & \textbf{85.4} & 46.7 & \textbf{50.8} & 48.8 & 23.2 & \underline{31.1} & 43.2 & 7.9 \\
\rowcolor[RGB]{236,244,252}
\quad + \ourapproach (Ours) & \textbf{19.0} & \textbf{20.8} & \textbf{56.0} & \textbf{72.3} & \underline{85.2} & \underline{52.3} & \underline{50.3} & \textbf{51.4} & \textbf{26.9} & \textbf{32.6} & \textbf{46.7} & \underline{6.2} \\
\midrule
\textbf{Qwen2.5-7B-Instruct} & 7.5 & 9.2 & 50.0 & 74.1 & 89.4 & 18.2 & 29.6 & 15.3 & 7.8 & 6.4 & 30.7 & $-$ \\
\quad + Only SFT & 14.6 & 14.1 & 56.2 & 71.3 & 89.6 & 33.3 & 37.9 & 37.1 & 14.6 & 23.2 & 39.2 & 8.2 \\
\quad + SFT-then-RL & \textbf{25.9} & 23.3 & 59.4 & \underline{77.4} & \underline{91.2} & \textbf{58.4} & \underline{53.6} & \underline{52.3} & \underline{27.0} & 30.7 & \underline{49.9} & \underline{7.7} \\
\quad + Zero-RL & 21.3 & \underline{27.1} & \underline{62.1} & 76.8 & 90.6 & 55.7 & 49.6 & 50.1 & 24.7 & \underline{31.5} & 49.0 & 7.9 \\
\rowcolor[RGB]{236,244,252}
\quad + \ourapproach (Ours) & \underline{23.2} & \textbf{29.3} & \textbf{64.8} & \textbf{80.4} & \textbf{92.1} & \underline{56.3} & \textbf{54.2} & \textbf{56.0} & \textbf{32.2} & \textbf{33.5} & \textbf{52.2} & \textbf{7.5} \\
\midrule
\textbf{Llama3.1-8B-Instruct} & 0 & 1.6 & 25.0 & 46.6 & 78.3 & 16.9 & 38.4 & 12.1 & 13.3 & 6.3 & 23.9 & $-$ \\
\quad + Only SFT & 7.5 & 11.2 & 34.3 & 59.3 & 85.6 & 32.9 & 41.1 & 38.2 & 15.9 & 21.7 & 34.8 & 8.4 \\
\quad + SFT-then-RL & \underline{16.0} & \underline{12.5} & 52.0 & \underline{64.0} & 86.1 & 53.8 & \textbf{55.4} & \underline{53.8} & \textbf{30.2} & \textbf{33.1} & \underline{45.7} & \textbf{7.1} \\
\quad + Zero-RL & \textbf{17.2} & \underline{12.5} & \underline{53.7} & 62.5 & \underline{87.9} & \underline{54.3} & 50.9 & 49.6 & 28.5 & 31.6 & 44.9 & 7.7 \\
\rowcolor[RGB]{236,244,252}
\quad + \ourapproach (Ours) & 15.5 & \textbf{20.0} & \textbf{58.4} & \textbf{65.2} & \textbf{88.2} & \textbf{57.4} & \underline{55.3} & \textbf{55.6} & \underline{29.7} & \underline{32.7} & \textbf{47.8} & \underline{7.4} \\
\bottomrule
\end{tabular}
\end{table*}

%% file: tables/main_table_compare_method.tex
\begin{table*}[!t]
\centering
\caption{Overall results on 10 challenging reasoning tasks. The top two results are highlighted in \textbf{bold} and \underline{underlined}. All baseline methods use Qwen2.5-3B as the backbone. For baselines lacking a 3B Instruct version, we replicate them using the exact parameter settings from their open-source code.}
\label{tab:main_table}
\setlength\tabcolsep{2.5pt}       % 既然少了两列，可以把间距稍微调大一点 (原为1pt)
\renewcommand{\arraystretch}{1}
\fontsize{8.5pt}{11pt}\selectfont
\begin{tabular}{p{3.25cm}ccccccccccc}   % 共 12 列
\toprule
\multirow{2}[2]{*}{\textbf{Method}} &
\multicolumn{5}{c}{\textbf{Computational Reasoning}} &  % 改为5列
\multicolumn{5}{c}{\textbf{Knowledge-Intensive Reasoning}} & % 改为5列
\multirow{2}[2]{*}{\textbf{Avg.}} \\
\cmidrule(lr){2-6} \cmidrule(lr){7-11} % 调整划线范围
 & AIME24 & AIME25 & AMC23 & GSM8K & MATH & 2Wiki. & Bamb. & HQA & MuSi. & SimQA & \\ % 删除了 AMC23 和 PopQA
\midrule
Qwen2.5-3B-Instruct
 & 4.5	& 8.6	&38.3	&65.3	&81.0 &	19.2	&27.5&	11.5&	6.3&	5.9 & 26.8\\
\midrule
\multicolumn{12}{l}{\textit{\textbf{LLM-based Math Agent}}} \\ % 改为12
ToRL
 & 10.0	&20.0	& 53.0	& \underline{72.0}	&84.4&	24.0&	25.4	&17.9&	8.3&	6.5 &	32.2\\
SimpleTIR
 & 17.5& 	19.1& 	\underline{54.3}	& 69.4	& 83.9& 	16.1& 	24.2& 	16.5& 	7.2	& 8.8	& 31.7\\
\midrule
\multicolumn{12}{l}{\textit{\textbf{LLM-based Search Agent}}} \\ % 改为12
Search-o1
 & 8.6	&10.3	&36.5&	58.6&	74.2&	27.4&	38.5	&29.5	&12.5	&10.8	&30.7\\
Search-R1
 & 3.3	 &10.0	 &31.2 &	58.0 &	60.5	 &44.9	 &44.2	 &50.1	 &22.2	 &\underline{31.4}	 &35.6\\
Tree-GRPO
 & 9.2	& 16.6& 	32.5& 	52.3& 	56.5& 	44.1& 	43.7	& 50.5	& 26.7	& 29.0 &36.1\\
\midrule
\multicolumn{12}{l}{\textit{\textbf{Multi-Tool Integrated Agent}}} \\ % 改为12
ReCall
 & 7.6	&12.4	&48.2	&68.6&	84.1&	42.6&	43.7	&48.7	&22.6	&26.5&	40.5\\
Tool-Star
 & \textbf{20.1}	&12.9	& 52.0	&69.8& 85.0 &	48.6	&44.1&	\textbf{52.4} &	\textbf{27.8} &	30.3 &	44.3\\
ARPO
 &  16.2	& \underline{20.6} &	52.0	&70.2&	\textbf{85.7}&	\textbf{54.3} &	\underline{47.6}	&50.0	&26.5	&29.3&	\underline{45.2}\\
\rowcolor[RGB]{236,244,252}
\ourapproach (Ours)
 & \underline{19.0} &	\textbf{20.8} &	\textbf{56.0}	 & \textbf{72.3} & \underline{85.2} &	\underline{52.3} &	\textbf{50.3}	 & \underline{51.4} &	\underline{26.9}	 &\textbf{32.6} & \textbf{46.7}\\
\bottomrule
\end{tabular}
\end{table*}